\UseRawInputEncoding

\documentclass[lettersize,journal]{IEEEtran}
\usepackage{amsmath,amsfonts}
\usepackage{algorithmic}
\usepackage{algorithm}
\usepackage{array}
\usepackage[caption=false,font=normalsize,labelfont=sf,textfont=sf]{subfig}
\usepackage{textcomp}
\usepackage{stfloats}
\usepackage{url}
\usepackage{xcolor}
\usepackage{verbatim}
\usepackage{graphicx}
\usepackage{cite}
\usepackage{amsmath,amsthm}
\newtheorem{corollary}{Corollary}
\newtheorem{thm}{Theorem}
\newtheorem{mydef}{Definition}
\begin{document}
%
\title{Rethinking Class Imbalance in Machine Learning}
%
%

\author{ Ou Wu
\thanks{Ou Wu is with the National Center for Applied Mathematics and School of Mathematics, Tianjin University, Tianjin,
China, 300072.
E-mail: wuou@tju.edu.cn
}
\thanks{Manuscript received May 05, 2023.}}

\markboth{Journal of \LaTeX\ Class Files,~Vol.~14, No.~8, August~2021}%
{Shell \MakeLowercase{\textit{et al.}}: A Sample Article Using IEEEtran.cls for IEEE Journals}


\maketitle

\begin{abstract}
Imbalance learning is a subfield of machine learning that focuses on learning tasks in the presence of class imbalance. Nearly all existing studies refer to class imbalance as a proportion imbalance, where the proportion of training samples in each class is not balanced. The ignorance of the proportion imbalance will result in unfairness between/among classes and poor generalization capability. Previous literature has presented numerous methods for either theoretical/empirical analysis or new methods for imbalance learning. This study presents a new taxonomy of class imbalance in machine learning with a broader scope. Four other types of imbalance, namely, variance, distance, neighborhood, and quality imbalances between/among classes, which may exist in machine learning tasks, are summarized. Two different levels of imbalance including global and local are also presented. Theoretical analysis is used to illustrate the significant impact of the new imbalance types on learning fairness. Moreover, our taxonomy and theoretical conclusions are used to analyze the shortcomings of several classical methods. As an example, we propose a new logit perturbation-based imbalance learning loss when proportion, variance, and distance imbalances exist simultaneously. Several classical losses become the special case of our proposed method. Meta learning is utilized to infer the hyper-parameters related to the three types of imbalance. Experimental results on several benchmark corpora validate the effectiveness of the proposed method.

\end{abstract}




%

\section{Introduction}
\IEEEPARstart{D}{ata} imbalance exists ubiquitously in real machine learning tasks. For instance, in object classification, the number of training samples for common objects like cups and buildings is often much greater than that of rare objects. The classes dominate the training set are referred to as majority classes, whereas those occupy little are called minority classes. In tasks with extreme class imbalance, also known as long-tailed classification~\cite{DTL}, the majority classes are referred to as ``head", while the minority classes are referred to as ``tail". The ignorance of the imbalance among classes will result in unfairness and even poor generalization capability.

To enhance the fairness among classes and increase the generalization capability, a number of studies involve the learning for class imbalance and constitute an independent research area of machine learning, namely, imbalance learning. Various classical methods have been proposed in the literature, such as logit adjustment~\cite{LA2021}, BBN~\cite{BBN2020}, MetaWeight~\cite{meta_class_weight}, LDAM~\cite{Cao2019}, and ResLT~\cite{ResLT}. Several benchmark data datasets have been compiled for evaluation. Despite the progress made in imbalance learning, addressing class imbalance encounters the following issues:
\begin{itemize}
    \item Previous research on imbalance learning has mainly focused on the imbalance in class proportions. However, there are other types of imbalances that have received little attention in the literature. Our theoretical investigation reveals that ignoring these other types of imbalances can impede our ability to effectively tackle machine learning tasks and utilize existing imbalance learning algorithms.
    \item The current approaches to imbalance learning solely focus on global imbalance, which considers the imbalance between/among entire classes. However, there is a notable imbalance within the local regions of classes that has scantily been considered in previous literature. It is imperative not to overlook imbalance within local regions, as neglecting it can lead to unfairness and suboptimal generalization capability.
\end{itemize} 

This study provides a comprehensive exploration of imbalance learning that goes beyond the scope of existing studies. First, four other types of class imbalance, namely, variance, distance, neighborhood, and quality, are introduced and formulated. The first three types of imbalance have been not referred to in previous literature. Although the fourth type is usually considered in noisy-label learning, it has not been explicitly recognized as a type of class imbalance\footnote{As quality imbalance is actually explored in noisy-label learning, it is not the focus of this study. In addition, some recent studies~(e.g.,~\cite{lin2017focal}) highlight that the different classes may contain different proportions of hard samples, which is also a form of quality imbalance.}. Further more, this study introduces the concept of imbalance from the viewpoint of the local perspective. Several research studies that propose intra-class imbalance can be considered examples of local imbalance. A series of theoretical analysis is then performed to quantify the influence of variance and distance imbalances as well as mixed imbalance. Our results demonstrate that even when there is no proportion imbalance, variance or distance imbalance can lead to an equivalent degree of unfairness. Based on our findings, we design a novel logit perturbation-based imbalance learning approach that improves upon existing classical methods. Our proposed method encompasses several classical methods as special cases. The effectiveness of our approach is validated by experiments carried out on benchmark data sets.

Our contributions can be summarized as follows:
\begin{itemize}
    \item  The scope of imbalance learning is expanded, and a more comprehensive taxonomy is developed for it. As far as we know, this study is the first to introduce the concepts of variance, distance, neighborhood, quality imbalance, and global/local imbalance.
    \item Theoretical analysis is conducted to quantify how variance and distance imbalances negatively affect model fairness. The case when more than one types of imbalance is also theoretically investigated. The conclusions enhance our understanding of class imbalance and classical methods. For instance, some studies report conflicting findings on the effectiveness of resampling-based imbalance learning~\cite{Megahed,Goorbergh}. Our analysis suggests current sampling-based methods only account for proportion imbalance, potentially yielding suboptimal results when other types of imbalances co-occur.
    \item A new logit perturbation-based imbalance learning method is proposed which can address not only the conventional proportion imbalance but also other types of imbalance (i.e. variance and distance) we discovered in our research. Our method can also derive several classical methods.
\end{itemize} 

The paper is organized as follows. Section II briefly reviews related studies. Section III introduces our constructed taxonomy for class imbalance and provides theoretical analysis. Section IV presents a new method that addresses multiple types of class imbalance. Section V presents experiments and discussions. The conclusion is provided in Section VI.

\section{Related work}
\subsection{Imbalance Learning}
Imbalanced learning is concerned with the fairness and generalization capability that occurs due to the class imbalance present among different categories in classification. Therefore, even if the class proportions in the test data are imbalanced, it is essential to employ imbalanced learning methods when fairness is required. Long-tailed classification, as a special case of imbalanced learning issues, has received increasing attention in recent years~\cite{DTL, Zhang2021survey}. Typical deep imbalance learning methods can be classified into the following folds:
\begin{itemize}
    \item Data resampling. The data resampling methodology proposed by Liu et al.~\cite{Liu2008} constructs a new training set by resampling the raw training data with a relatively low sampling rate for the majority classes and a higher rate for the minority classes. However, experimental comparison shows that this strategy is inefficient in many tasks.
    \item New loss. This type of methods varies the training loss based on the use of sample reweighting~\cite{Zhang2021data}, sample perturbation~\cite{LA2021}, or other data-driven approaches~\cite{Huang2781}. In the case of reweighting, large weights are assigned to samples from minority categories, while in perturbation, the samples from minority categories are perturbed to increase the loss. Dong et al.~\cite{Dong1869} designed a novel class rectification loss to avoid the dominant effect of majority classes. Li et al.~\cite{Li4812} established a new loss which can regularize the key points strongly to improve the generalization performance and assign large margin on tail classes. 
    \item New network. This type of methods designs more sophisticated networks for imbalanced tasks. For example, Zhou et al.~\cite{BBN2020} developed a bilateral-branch network balance representation and classifier training, leading to effective feature representations for both head and tail categories. Additionally, Cui et al.~\cite{ResLT} designed a novel residual fusion mechanism that includes three branches to optimize the performance of the head, medium, and tail classes.
    \item Data augmentation. This type of methods generates new training data to address class imbalance. Zhang et al.~\cite{BAG2021} proved that mixup~\cite{Verma2019}, a common data argumentation technique, is effective in dealing with long-tailed classification. Wang et al.~\cite{DGM} developed a novel generative model for effective data generation for minor categories. Jing et al.~\cite{Jing139} divided the training data into multiple subsets and proposed a sophisticated strategy for the successive learning, which can partially be viewed as a form of data augmentation.
\end{itemize}

The above-mentioned studies aim to address the issue of inter-class proportion imbalance. Recently, several pioneering studies are conducted to investigate other imbalance settings. Tang et al.~\cite{Tang2022} firstly explored attribute-wise intra-class imbalance in which samples within each class are also imbalanced due to the varying attributes. Liu et al.~\cite{Liu2021} firstly explored difficulty-aware intra-class imbalance, where samples with different difficulty levels within each class are imbalanced. Additionally, some studies were carried out for imbalance regression\cite{Ren2022,Yang2021}. Oksuz et al.~\cite{Oksuz2021} presented a comprehensive survey on the imbalance learning issue in object detection and identified three other types of imbalance in object detection.
\begin{figure}[b] 
    \centering
    \includegraphics[width=1\linewidth]{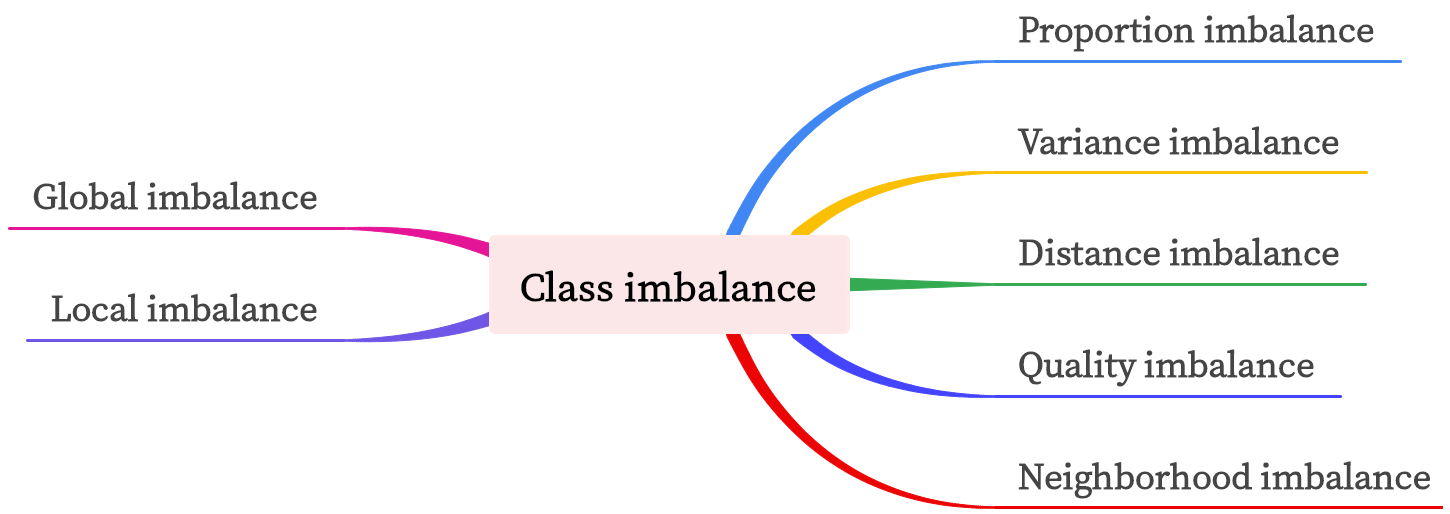} 
    \caption{The proposed new taxonomy for class imbalance. Existing studies merely deal with proportion imbalance.}
\end{figure}
\subsection{Logit Perturbation}
Our previous study~\cite{LMY2022} showed that many classical methods with different inspirations can be attributed to the perturbation on logits such as logit adjustment (LA)~\cite{LA2021}, LDAM~\cite{Cao2019}, and ISDA~\cite{Wang2022}. Let $f(x_i)$ be the logit output by a deep neural network for a sample $x_i$, and let $y_i$ be corresponding label. Class-wise logit perturbation modifies the standard cross-entropy loss into the following  
\begin{equation}
    l(y_i,f(x_i)) = -{log}\frac{e^{[f_{y_i}(x_i)+\delta_{{y_i}{y_i}}]}}{e^{[f_{y_i}(x_i)+\delta_{{y_i}{y_i}}]}+\sum_{y'\neq {y_i}}e^{[f_{y'}(x_i)+\delta_{{y_i}y'}}]},
\end{equation}
where $l$ is the loss and $\delta_{y_i}$ is the perturbation vector for the logits of the ${y_i}$th class. 

According to our previous analysis in Ref.~\cite{LMY2022}, one can improve the accuracy of a specific category by increasing the loss when perturbing the logits of that category.  Both LA and LDAM follow this guideline by increasing the losses of the minority classes to a greater extent than those of the majority classes. ISDA, on the other hand, does not adhere to this guideline and was unsuccessful in classifying long-tail datasets. This guideline can be interpreted as tuning the margins between classes. Categories with low accuracy usually have small margins, so increasing the margins for these categories will increase their accuracy.

\section{New Taxonomy for Class Imbalance}
In this section, a new taxonomy for class imbalance is firstly presented. Theoretical analyses are then conducted to verify the reasonableness of the taxonomy. Some symbols and notations are described at first. Let $S = \{ x_i, y_i\}_{i=1}^N$ be a set of $N$ training samples, where $x_i$ is the feature and $y_i$ is the label. Let $C$ be the number of categories and $\pi_c{\rm{ }} = {\rm{ }}N_c/N$ be the proportion of the samples, where $N_c$ is the number of the samples in the $c$th category in $S$. In addition, let $p_c$ and $p(x|y=c)$ be the prior and the class conditional probability density for the $c$th class, respectively. Let $\Sigma_c$ be the co-variance matrix for the $c$th class. When there is no ambiguity, $x_i$ represents the feature output by the last feature encoding layer. $\mathcal{E}\left(f,y\right)$ be the classification error of $f$ on class $y$.

\begin{figure*}[t] 
    \centering
    \includegraphics[width=1\linewidth]{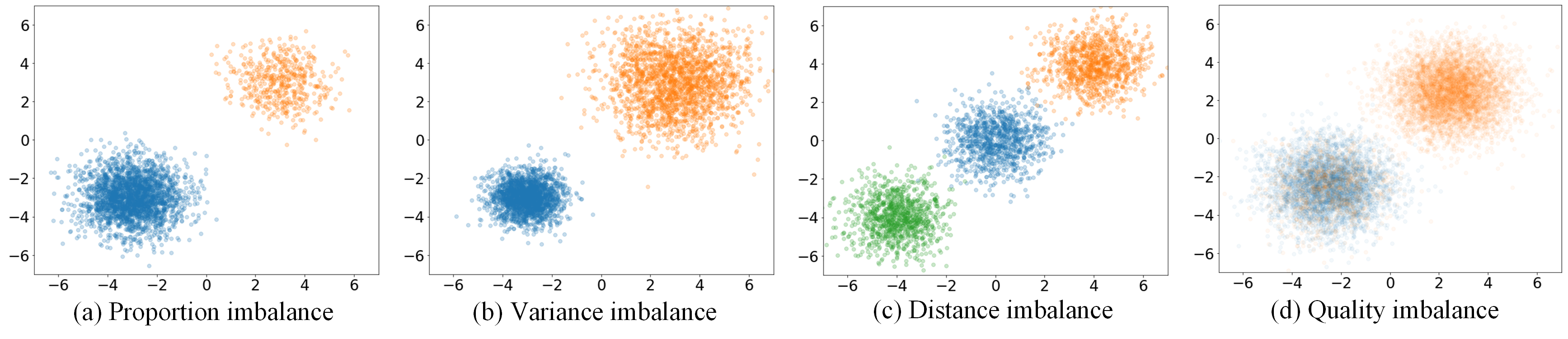} 
        \vspace{-0.3in}
    \caption{Illustrative examples for the four types of imbalance. In (a), the two classes have different proportions of samples; in (b), the two classes have different co-variance matrices; in (c), the middle class has the smallest average inter-class distance; in (d), the two classes have different noisy-label rates.}
    \label{fig3}
\end{figure*}

\subsection{New Taxonomy}
Fig.~1 presents a new taxonomy for class imbalance in machine learning. The taxonomy includes two independent divisions for class imbalance, as shown in the figure. The right division categorizes class imbalance into proportion, variance, distance, neighborhood, and quality imbalances. The left division categorizes class imbalance into global and local imbalances.

\begin{figure}[h] 
    \centering
    \includegraphics[width=0.9\linewidth]{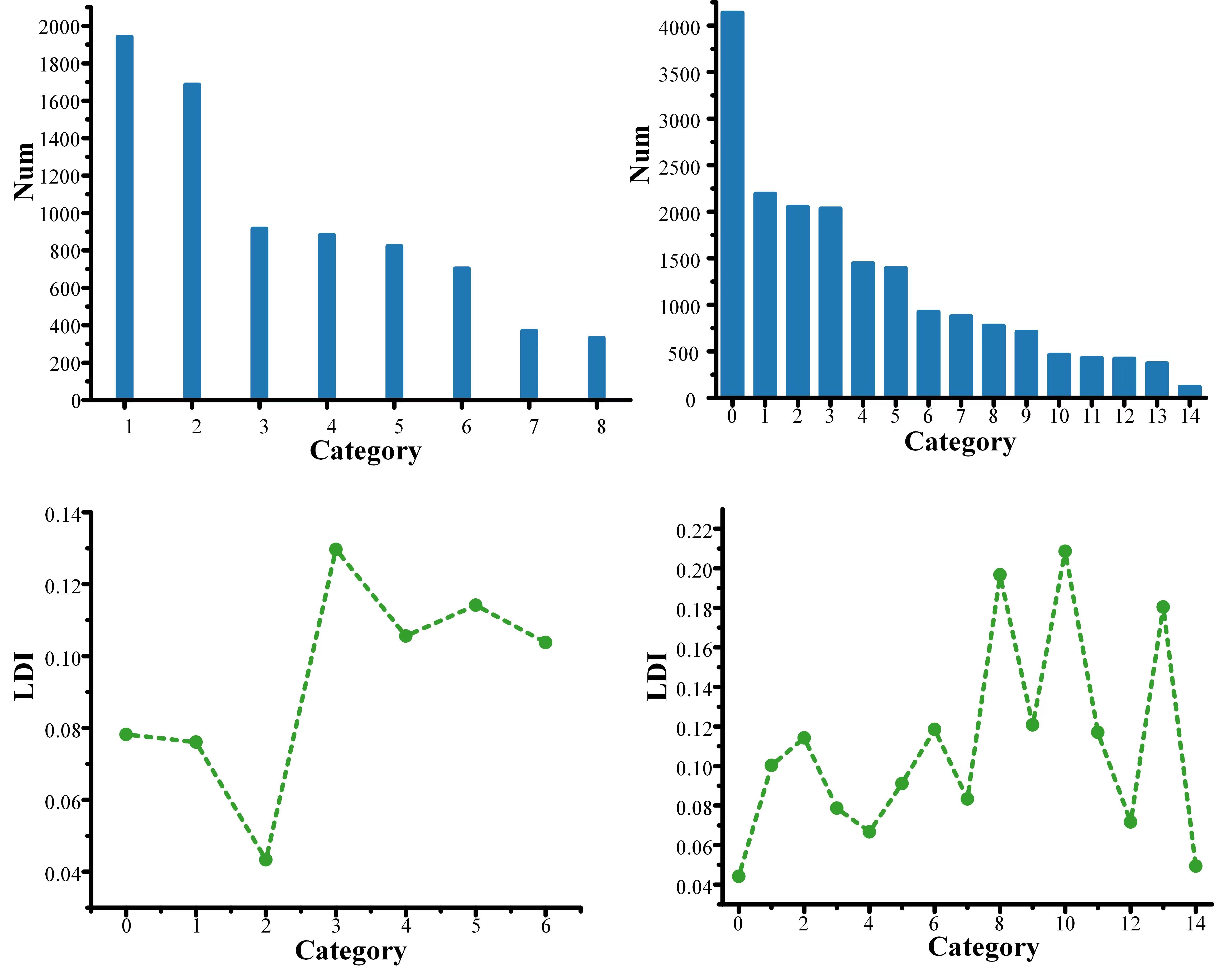} 
    \vspace{-0.18in} 
    \caption{The numbers (up) and the average LDI indexes of categories on two benchmark graph learning data sets~\cite{WangIJCNN}. A large LDI value of a node denotes that there is a large proportion of neighborhood nodes from different classes.}
    \label{fig3}
\end{figure}

First, the right division of Fig.1 is described as follows:
\begin{itemize}
    \item \textbf{Proportion imbalance}. It means that $\pi_c{\rm{ }}$ are unequal among different classes. In most studies dealing with class imbalance, ``class imbalance" is synonymous with ``proportion imbalance".  Without loss of generality, we assume $\pi_1 \geq {\rm{ }} \cdots  \geq {\rm{ }}\pi_c \geq {\rm{ }} \cdots \geq \pi_C$. In extremely imbalanced tasks, the first few classes are referred to as ``head" classes, while the last few classes are referred to as ``tail" classes. Fig.~2(a) illustrates the proportion imbalance in which the dominant class is the bottom left class, evidenced by its significantly larger proportion compared to that of the upper right class.
    \item \textbf{Variance imbalance}. As shown in Fig.~2(b), although the two categories have equal proportions (or prior probabilities), their variances differ. The data points in the bottom left category are more tightly clustered than those in the upper right category. Variance differences can result in imbalance. However, it should be noted that differences in variance~(or co-variance matrix) do not necessarily lead to imbalance, which will be explained in the following subsection. In the next subsection, we will theoretically measure the variance imbalance and demonstrate its negative impact on fairness.
    \item \textbf{Distance imbalance}. Some classes are closer to the rest than others, and the difference between the means of two classes is used as a measure of class distance. The average inter-class distances for the three classes differ, as illustrated in Fig.~2(c). The middle class is likely to have the poorest classification performance because it has the smallest average inter-class distances. 
    \item \textbf{Quality imbalance}. Differences in sample quality due to factors such as data collection and data labeling can result in varying overall quality across classes or quality imbalances. Subsequent subsections will demonstrate that imbalance in quality from Gaussian noise is in fact variance imbalance, and quality imbalance from label noise is commonly researched in noisy-label learning. As such, quality imbalance is not a main focus of our study. Fig.~2(d) indicates that the blue class is affected by a higher rate of noisy labels than the orange class.
    \item \textbf{Neighborhood imbalance}. This type of imbalance is specific to graph node classification tasks. It means that certain categories of nodes have a greater proportion of heterogeneous nodes than others. Our previous study~\cite{WangIJCNN} measures the distribution of heterogeneous nodes in a node's neighborhood by introducing a new index called LDI. The larger the value of LDI of a node is, the more heterogeneous nodes will locate in the node's neighborhood. Fig.~3 demonstrates that tail classes have higher LDI values, and categories with similar numbers of training samples may still have varying LDI values (for example, categories 6 and 7 in the right corpus). An imbalance in the node's neighborhood can worsen the performance of certain tail classes with high average LDI values.
\end{itemize}

\begin{figure}[b] 
    \centering
    \includegraphics[width=1\linewidth]{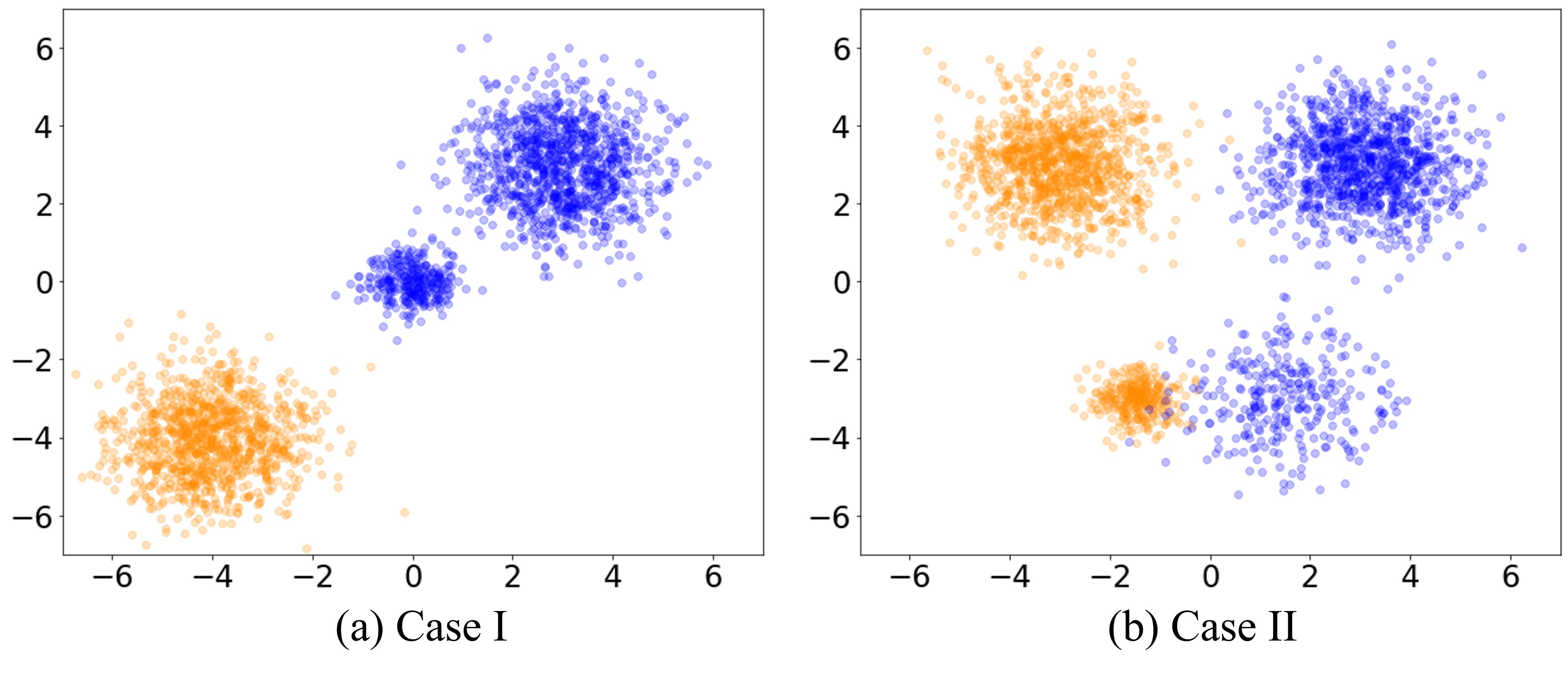} 
    \vspace{-0.3in} 
    \caption{Two typical cases of local imbalance.}
    \label{fig3}
\end{figure}

\begin{figure}[t] 
    \centering
    \includegraphics[width=0.68\linewidth]{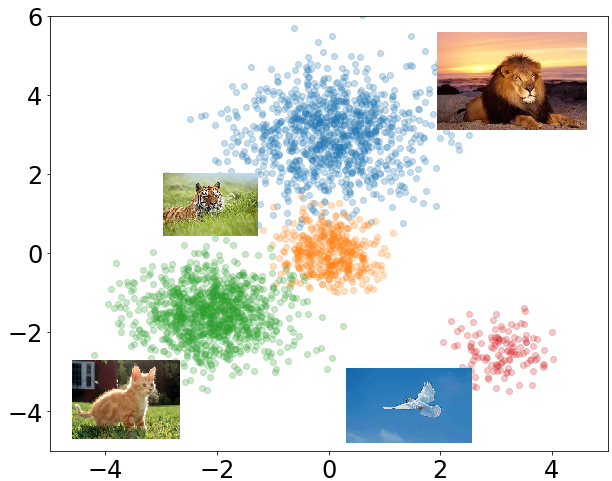} 
    \vspace{-0.15in} 
    \caption{The four classes have different proportions, covariance matrices, and average inter-class distances.}
    \label{fig3}
\end{figure}

Second, the left division of Fig.~1 is described as follows:
\begin{itemize}
    \item \textbf{Global imbalance}: It denotes the presence of imbalance across all classes. Imbalance in sample proportions between different classes is a common type of overall imbalance, as depicted in all four examples in Fig.~2. 
    \item \textbf{Local imbalance}: It  refers to the existence of imbalance in specific regions of some classes. There are at least two cases. The first case refers to that imbalance exists in the local areas of one class as shown in Fig.~4(a). The second case refers to that imbalance exists in the local areas of two classes as shown in Fig.~4(b). The two majority parts of the two classes are balanced. However, their two minority parts are imbalanced.
\end{itemize}

Obviously, the intra-class imbalance investigated in previous literature belongs to the first case of local imbalance. For example, in the difficulty-aware intra-class imbalance~\cite{Liu2021}, a class is divided into hard and easy areas; in attribute-wise imbalance~\cite{Tang2022}, a class can be divided into different areas according to attributes and imbalance can occur in these areas.

Additionally, mixed imbalance types may occur in real-life tasks, especially in multi-class issues. Fig.~5 shows a complex example of mixed imbalance where the four classes occupy different proportions and have different inter-class distances. Moreover, the variances of the tiger and the dove classes are smaller compared to those of the lion and the cat classes. Although the proportion of the dove class is minor, it is far from the other three classes. As a result, the minor proportion of the dove class will not negatively influence its performance.

\subsection{Theoretical Verification for the Taxonomy}
In this subsection, several typical learning tasks are designed to illustrate the newly presented imbalance types including variance, distance, quality, and local imbalances. One typical mixed case is also explored.

\subsubsection{Variance imbalance}
Considering the following binary learning task. The data from each class follow a Gaussian distribution $\mathcal{D}$ that is centered on $\boldsymbol{\theta}$ and $\boldsymbol{-\theta}$, respectively. A $K$-factor difference is found between two classes’ variances: $\sigma_{+1}: \sigma_{-1}=K:1$ and $K>1$. The data follow
\begin{equation} 
\begin{aligned}
y \stackrel{u.a.r}{\sim}\{-1,+1\}, \quad \boldsymbol{\theta}=[{\eta, \ldots, \eta} ]^T \in \mathbb{R}^d,\eta > 0,\\
\boldsymbol{x} \sim\left\{\begin{array}{ll}\mathcal{N}\left(\boldsymbol{\theta}, \sigma_{+1}^{2} \boldsymbol{I}\right), & \text { if } y=+1, \\ \mathcal{N}\left(-\boldsymbol{\theta}, \sigma_{-1}^{2} \boldsymbol{I}\right), & \text { if } y=-1.\end{array}\right.
\end{aligned}
\vspace{-0.03in}
\end{equation}

Variance difference exists in the task as $K \neq 1$. Let the performance gap ($\nabla_{err} = |\mathcal{E}\left(f,+1\right) - \mathcal{E}\left(f,-1\right)|$) be the classification error difference between classes `+1' and `-1' given a linear classifier $f$. We have the following theorem.
\begin{thm}
For the above binary task, let $f^*$ be the optimal linear classifier which minimizes the following classification error~\cite{Xu2021}\footnote{It should be pointed out that although this theorem is presented in~\cite{Xu2021}, the paper does not mention or discuss any concerns related to variance imbalances.}
\begin{equation}
f^*=\arg\underset{f}{ \min }\operatorname{Pr} (f(\boldsymbol{x}) \neq y). 
\end{equation}
Then $\nabla_{err} > 0$ and the class `+1' is harder.
\end{thm}

\begin{figure}[t]      \centering
    \includegraphics[width=0.68\linewidth]{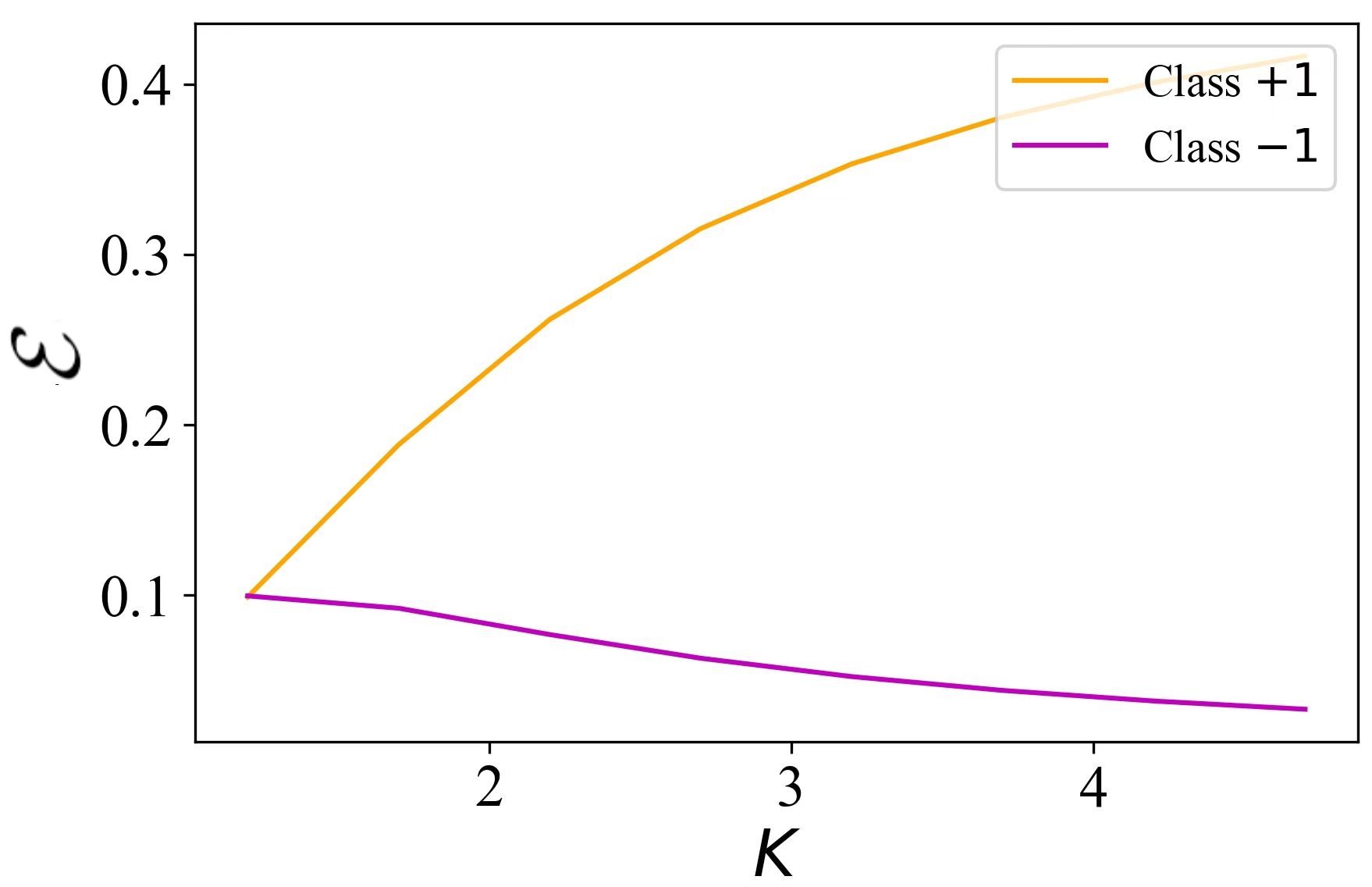} 
     \vspace{-0.15in} 
    \caption{The classification errors~($\mathcal{E}$) of the two classes with the increasing of $K$. The performance gap is increased.}
    \label{fig3}
\end{figure}

Theorem 1 verifies that variance imbalance can also lead to unfairness between classes. Fig.~6 shows an illustrative example. The performance gap between the two classes is increased with the increasing of $K$. As the variance of each class is represented by a co-variance matrix instead of a real value, a natural question arises that how to measure the variance imbalance between two classes. It is inappropriate to utilize the ratio of the norms of two matrices to measure variance imbalance. Although the covariance matrices and their norms differ between the two classes in the three cases shown in Fig.~7, there is no variance imbalance in any of these cases. In other words, the differences between the corresponding covariance matrices do not negatively impact any class, as evidenced by the fact that all three examples have the same class boundaries of $x_1=4$. This study proposes a measure based on data mapping/projection, as depicted in Fig.~8. 

\begin{figure}[t] 
    \includegraphics[width=1\linewidth]{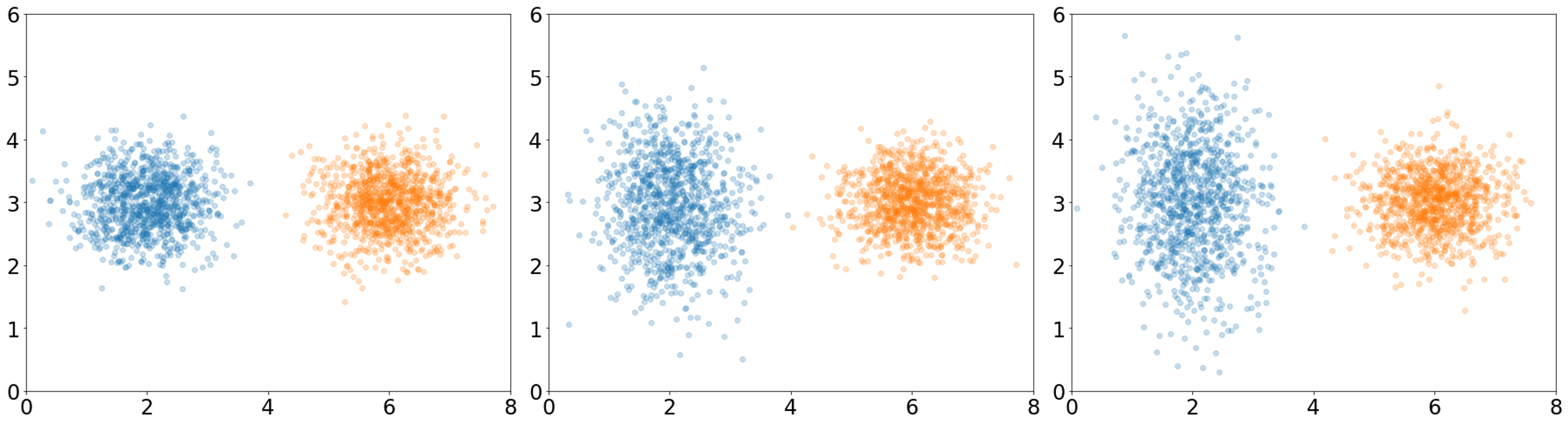} 
     \vspace{-0.25in} 
    \caption{Although the variance of the blue class varies from the left to the right, the optimal linear classifier remains unchanged.}
    \label{fig3}
\end{figure}

\begin{figure}[ht] 
    \includegraphics[width=1.02\linewidth]{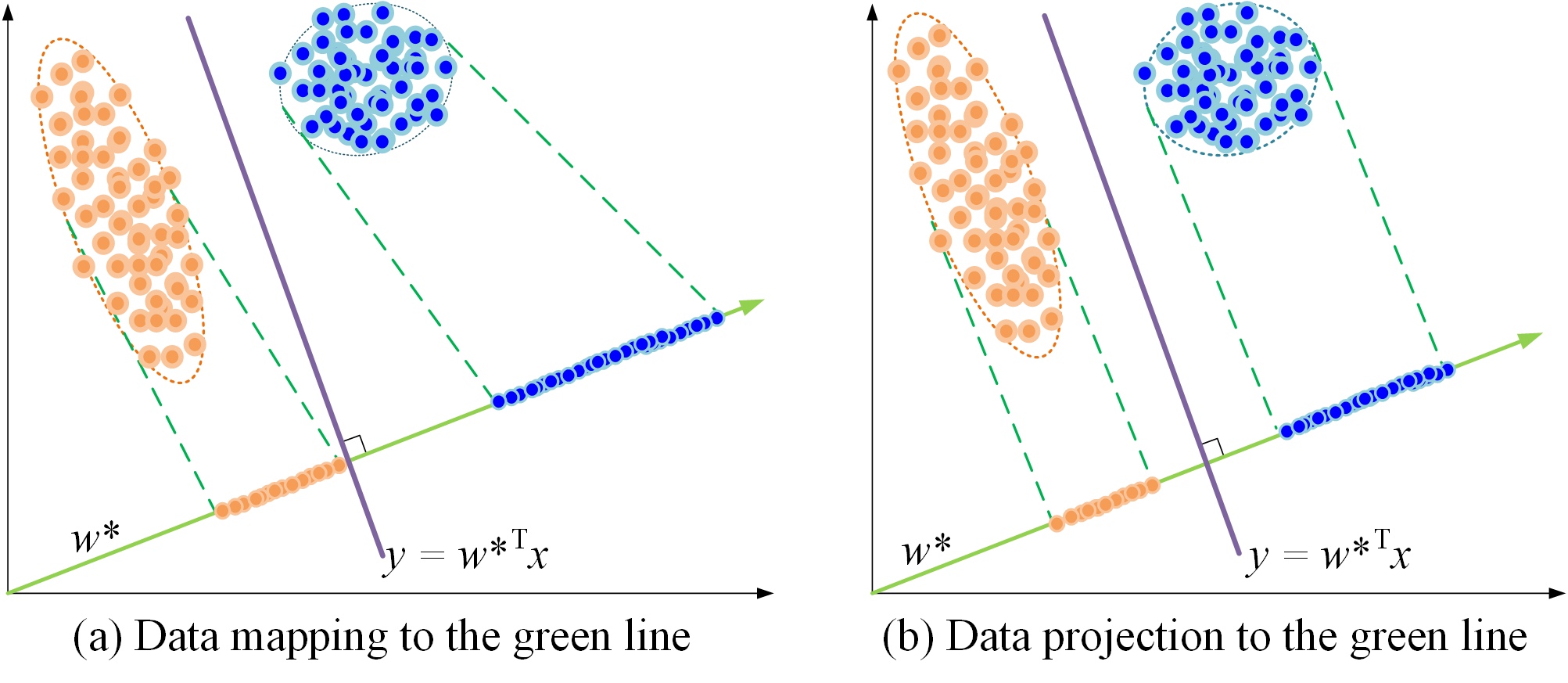} 
     \vspace{-0.3in} 
    \caption{The mapping (left) and projecting of data from two classes to the direction vector~(i.e., $w^*/||w^*||$) of the classifier function $y={w^*}^Tx$. The green line represents the direction vector.}
    \label{fig3}
\end{figure}

\begin{mydef}
\textbf{Measure for Variance imbalance}: Given a binary classification task in which the feature covariance matrices are $\sum_+$ and $\sum_-$, respectively. Let $w^*$ denote the coefficient of the optimal linear classifier. $w^*/||w^*||_2$ is the director vector of the linear function for the optimal classifier. Then, the variance imbalance between the two classes can be measured with the mapped/projected variance ratio to the line $y = {w^*}^Tx$ of the two classes as follows:
\begin{equation}
\nu=\frac{{w^*}^T\sum_+w^*/{{w^*}^T}{w^*}}{{w^*}^T\sum_-w^*/{{w^*}^T}{w^*}}=\frac{{w^*}^T\sum_+w^*}{{w^*}^T\sum_-w^*}  . 
\end{equation}
\end{mydef}

In fact, ${w^*}^T\sum_+w^*$ and ${w^*}^T\sum_-w^*$ are the mapped variances of the two classes as shown in Fig.~8(a), respectively; ${w^*}^T\sum_+w^*/{{w^*}^T}{w^*}$ and ${w^*}^T\sum_-w^*/{{w^*}^T}{w^*}$ are the projection variances of the two classes as shown in Fig.~8(b), respectively.
Take the orange class as an example. The mapped variance of the class means the variance of the mapped data on the green line of Fig.~8(a) of the class. As the mapped data are actually one-dimensional, the variance can be easily calculated. The projected variance of the class means the variance of the projected data on the green line of Fig.~8(b) of the class. The corresponding variance can also be easily calculated as the projected data are actually one-dimensional.

With the measure in Eq.~(4), the variance imbalance score between classes +1 and -1 in Eq.~(2) is as follows:
\begin{equation}
\nu=\frac{{w^*}^T\sum_+w^*}{{w^*}^T\sum_-w^*} =\frac{\sigma_{+1}^2}{\sigma_{-1}^2} = K^2. 
\end{equation}

Eq.~(4) can also measure the degree of variance imbalance in Fig.~7. The values of $\nu$ for the three cases in Fig.~7 are all equal to one, even though their covariance matrices are different, indicating that no variance imbalance exists in all three cases. For example, assume that $\sum_+ = [[2,0],[0,8]]$ and $\sum_- = [[2,0],[0,4]]$ in the left case of Fig.~9. Obviously, $w^* = [1,0]^T$, so $\nu = 1$, indicating no variance imbalance. In the right case of Fig.~9, $\sum_+ = [[8,0],[0,2]]$ and $\sum_- = [[2,0],[0,8]]$. Then, $\nu = 4$, indicating that there is variance imbalance.

\begin{figure}[t] 
    \includegraphics[width=1\linewidth]{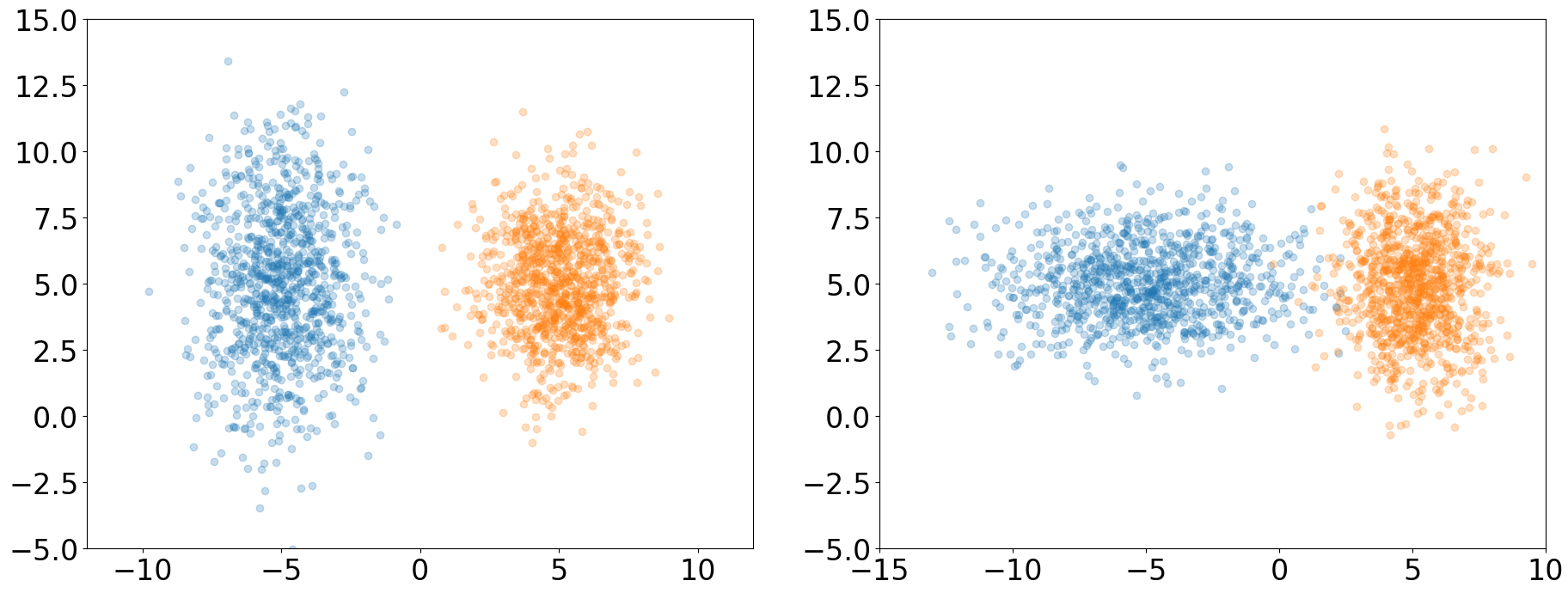} 
     \vspace{-0.25in} 
    \caption{Left: no variance imbalance; right: variance imbalance exists.}
    \label{fig3}
\end{figure}

In real applications, $w^*$ is usually unknown. Consequently, the degree of variance imbalance is measured based on a given linear classifier. For a binary learning task ($y \in \{0, 1\}$), let $a$ represent the final feature for a sample, and the logit of the sample is $v=w^Ta+b$, where $w = [w_0, w_1]^T$. It is easy to prove that $w_0-w_1$ is along with the direction vector of the underlying linear classifier. Let $\triangle w = w_0-w_1$. Inspired by previous work~\cite{Deng2019} that explored logit adjustment and the ArcFace loss, we define a new loss to alleviate the negative influence of variance imbalance as follows:
\begin{equation} 
\begin{aligned}
l(x,y)&=-log\frac{e^{w_y^Ta+b_y-\lambda {\triangle w}^T\sum_y \triangle w}}{e^{w_y^Ta+b_y-\lambda {\triangle w}^T\sum_y \triangle w}+e^{w_{1-y}^Ta+b_{(1-y)}}}\\
&=-log\frac{e^{w_y^Ta+b_y}}{e^{w_y^Ta+b_y}+e^{w_{1-y}^Ta+b_{(1-y)}+\lambda {\triangle w}^T\sum_y \triangle w}}
\end{aligned},
\end{equation}
which adds an additional class-wise margin to each sample and the margin equals to the mapped variance of the corresponding class. Obviously, if one class has larger mapped variance, then the added margin will be larger than that of the other class. A larger margin on the harder class will alleviate the unfairness incurred by variance imbalance. Eq.~(6) can be extended to the multi-class case ($y \in \{1,2,\cdots, C\}$) as follows:
\begin{equation} 
\begin{aligned}
l(x,y)&=-log\frac{e^{w_y^Ta+b_y}}{e^{w_y^Ta+b_y}+\sum_{c\neq y}e^{w_{c}^Ta+b_c+\lambda {\triangle w_{yc}}^T\sum_y \triangle w_{yc}}}
\end{aligned},
\end{equation}
where $\triangle w_{yc} = w_y - w_c$. Eq.~(7) is exactly the ISDA loss~\cite{Wang2022} which is inspired by the implicit semantic augmentation. Essentially, we provide an alternative interpretation for the ISDA which is actually a variance imbalance-aware logit perturbation method. Naturally, ISDA cannot cope well with proportion imbalance, which has been verified by existing studies and ISDA performs bad in benchmark long-tail datasets~\cite{LMY2022}.


\subsubsection{Distance imbalance}
In this study, the ``Distance" in distance imbalance particularly denotes inter-class distance which is defined as the distance between the class centers of two involved classes. We acknowledge that there might be a more appropriate metric to measure the relationship between two categories, but that is out of the scope of this study and we leave it for future work. In the following three-class learning task, the data from each class follow a Gaussian distribution $\mathcal{D}$ that is centered on $\boldsymbol{\theta}$, $\boldsymbol{0}$, and  $\boldsymbol{-\theta}$, respectively. Their covariance matrices and prior probabilities are identical. The data follow\vspace{-0.01in}
\begin{equation} 
\begin{aligned}
y \stackrel{u . a . r}{\sim}\{0,1,2\}, \quad \boldsymbol{\theta}=[{\eta, \ldots, \eta} ]^T \in \mathbb{R}^d,\eta > 0,\\
\boldsymbol{x} \sim\left\{\begin{array}{lll}\mathcal{N}\left(\boldsymbol{\theta}, \sigma^{2} \boldsymbol{I}\right), & \text { if } y=0, \\
\mathcal{N}\left(\boldsymbol{0}, \sigma^{2} \boldsymbol{I}\right), & \text { if } y=1,\\
\mathcal{N}\left(-\boldsymbol{\theta}, \sigma^{2} \boldsymbol{I}\right), & \text { if } y=2.\end{array}\right.
\end{aligned}
\vspace{-0.03in}
\end{equation}
In this task ($d=2$) as the average inter-class of class `1' is $\sqrt{2}\eta$ and those of classes `0' and `2' are $(\sqrt{2}+\sqrt{5})\eta/2$. Class `1' is more closer to the rest classes than other classes as illustrated in Fig.~1(c). As a consequence, only the distance imbalance exists. This imbalance in distance can also result in unfairness, as proven in the following theorem.

\begin{thm}
For the above classification task, let $f^*$ be the Bayes optimal classifier which minimizes the following classification error
\begin{equation}
f^*=\arg\underset{f}{ \min }\operatorname{Pr} (f(\boldsymbol{x}) \neq y). 
\end{equation}
Then the classification accuracy for the three classes is:
\begin{equation}
\begin{aligned}
&Acc(f^*,0) = 1-\operatorname{Pr}\{\mathcal{N}(0,1) \leq -\frac{3\sqrt{d}\eta}{2\sigma}\}, \\
&Acc(f^*,1) = 1-2*\operatorname{Pr}\{\mathcal{N}(0,1) \leq -\frac{3\sqrt{d}\eta}{2\sigma}\}, \\
&Acc(f^*,2) = 1-\operatorname{Pr}\{\mathcal{N}(0,1) \leq -\frac{3\sqrt{d}\eta}{2\sigma}\},
\end{aligned}
\end{equation}
where $\mathcal{N}(0,1)$ is the standard normal distribution. Obviously, class `1' is the hardest and has the lowest classification accuracy.
\end{thm}
\begin{proof} To achieve the Bayes optimal classifier, the classification rule for $f^*$ between classes `0' and the rest two classes is
\begin{equation}
\begin{aligned}
&\text{if} \quad \text{Pr}(y=0|x) > \text{Pr}(y=1|x), \text{Pr}(y=2|x) \\
&\text{then} \quad  f^*(x) = 0
\end{aligned}.
\end{equation}
To satisfy $\text{Pr}(y=0|x) > \text{Pr}(y=1|x)$, we have
\begin{equation}
\begin{aligned}
&e^{(x-\boldsymbol{\theta})^T\Sigma_{0}(x-\boldsymbol{\theta})} > e^{x^T\Sigma_{1}x}.
\end{aligned}
\end{equation}
Note that $\Sigma_0 = \Sigma_1 = \sigma^{2} \boldsymbol{I}$. The following inequality is derived:
\begin{equation}
\begin{aligned}
\boldsymbol{\theta}^Tx-\frac{\boldsymbol{\theta}^T\boldsymbol{\theta}}{2} > 0.
\end{aligned}
\end{equation}
Then the classifier boundary is linear. Let $f^*(x) = w^{*T}x + b^*$. If $w^*$ is set as $\frac{\boldsymbol{\theta}}{\eta} (= [1,\cdots, 1]^T)$, then we have
\begin{equation}
\begin{aligned}
b^* =-\frac{\boldsymbol{\theta}^T\boldsymbol{\theta}}{2\eta} = -\frac{{d}\eta^2}{2\eta}= -\frac{{d}\eta}{2}.
\end{aligned}
\end{equation}
Likewise, the classifier boundary between classes `1' and `2' can also be obtained. The optimal Bayes classifier is as follows:
\begin{equation} 
\begin{aligned}
f^*(x) = \left\{\begin{array}{lll}0, & \text { if } {w^*}^Tx-\frac{{d}\eta}{2} > 0, \\
2, & \text { if } {w^*}^Tx+\frac{{d}\eta}{2} < 0,\\
1, & \text { otherwise } \end{array}\right.
\end{aligned}
\vspace{-0.03in}
\end{equation}
Accordingly, the classification accuracy of $f^*$ on class `0' is $Acc(f^*,0) = 1-\text{Pr}\{\mathcal{N}(0,1) \leq -\frac{3\sqrt{d}\eta}{2\sigma}\}$. With the similar steps, the classification accuracy for the rest two classes can also be derived.
\end{proof}
\begin{figure}[t] %
\centering
    \includegraphics[width=0.6\linewidth]{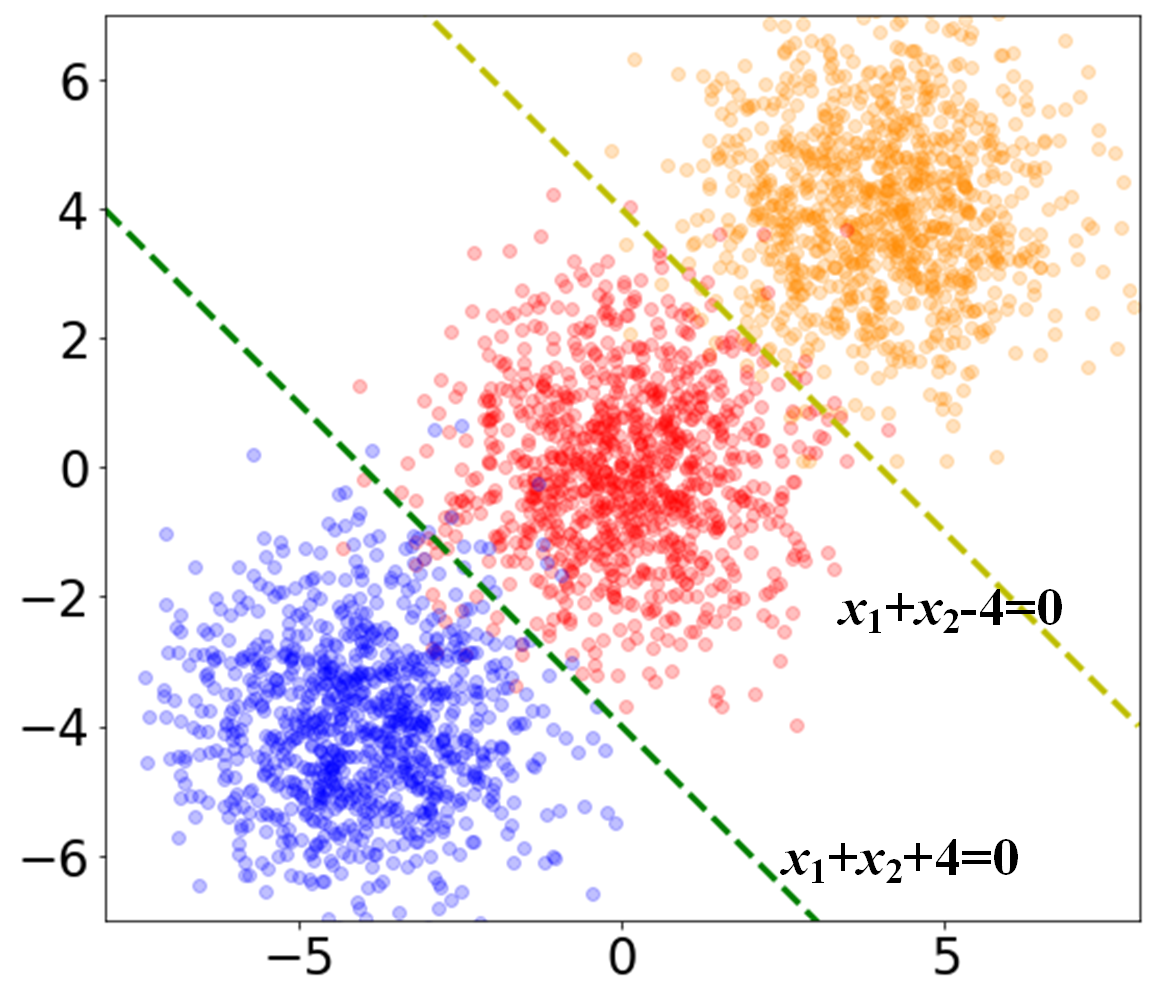} 
     \vspace{-0.15in} 
    \caption{Three classes and two classifier boundaries.}
    \label{fig3}
\end{figure}

Fig.~10 illustrates an example with $d=2$ and $\eta = 4$. According to Theorem 2, unfairness among classes can occur even if proportion and variance imbalances do not exist but distance imbalance does. The performance gap between class 0 and 1 is defined by $\text{Pr}\{\mathcal{N}(0,1) \leq -\frac{3\sqrt{d}\eta}{2\sigma}\}$. The performance gap can be minimized by reducing the class variance $\delta$ (e.g., Center loss~\cite{Wen2016}) or by increasing the class distance $\eta$ (e.g., Island loss~\cite{Cai2018}). In many multi-class learning tasks, distance imbalance exists inevitably. Let $\mu_c$ be the center of the $c$th class to be learned. Hayat et al.~\cite{Hayat2019} defined the following regularization term to ensure equidistant class centers:
\begin{equation}
\begin{aligned}
Reg(f) = \sum_{c<j}(||\mu_c-\mu_j||_2^2-u)^2, \\
u = \frac{2}{C^2-C}\sum_{c<j}||\mu_c-\mu_j||_2^2,
\end{aligned}
\end{equation}
where $u$~($\geq 0$) is the average inter-class distance. Eq.~(16) actually aims to directly reduce the distance imbalance by penalizing large or small inter-class distances.

The imbalance of attributes between classes can lead to distance imbalance. When some attribute values are nearly the same in the head and tail classes, the class distance between them becomes smaller, resulting in the problem of distance imbalance. Tang et al.~\cite{Tang2022} proposed a modified center loss to address attribute-wise imbalance and it outperforms existing methods.

\subsubsection{Quality imbalance}
Data quality in this study refers to the feature quality and the label quality incurred by noise. We first show that the imbalance incurred by feature noise is actually a case of variance imbalance.

The binary learning task investigated in Section III-B~(Eq.~(2)) is still adopted with the constraint that there is no variance imbalance (i.e., $K=1$). Nevertheless, feature noise exists. Assuming that the feature noise of the two classes follows $\mathcal{N}\left(\boldsymbol{0}, \epsilon_1^{2} \boldsymbol{I}\right)$ and $\mathcal{N}\left(\boldsymbol{0}, \epsilon_2^{2} \boldsymbol{I}\right)$, respectively. The feature distribution becomes
\begin{equation} 
\begin{aligned}
\boldsymbol{x} \sim\left\{\begin{array}{ll}\mathcal{N}\left(\boldsymbol{\theta},(\sigma^{2}+\epsilon_1^{2}) \boldsymbol{I}\right), & \text { if } y=+1, \\ \mathcal{N}\left(-\boldsymbol{\theta}, ( \sigma^{2}+\epsilon_2^{2}) \boldsymbol{I}\right), & \text { if } y=-1,\end{array}\right.
\end{aligned}
\end{equation}
which denotes that variance imbalance occurs if $\epsilon_1^{2} \neq \epsilon_2^{2}$. Therefore, quality imbalance in terms of feature noise is not further investigated in this study.

The issue for the imbalance incurred by label noise is actually the learning under asymmetric label noise, which has been widely investigated in previous literature~\cite{Scott2013}. Take the binary learning task as an example. Let $r_0 = \operatorname{Pr}(\Tilde{y}=-1|y=+1)$ and $r_1 = \operatorname{Pr}(\Tilde{y}=+1|y=-1)$ be the two noise rates. When $r_0 \neq r_1$, asymmetric label noise exists. In other words, quality imbalance in terms of label noise occurs, which will result in fairness between the classes and the class with a high noisy rate will have a larger classification error. Gong et al.~\cite{Gong2022} proposed the use of two virtual auxiliary sets to correct the labels of false negative and false positive samples separately. Asymmetric label noise has also been extensively studied~\cite{Han2021}, and will therefore not be discussed further in this paper.

\subsubsection{Neighborhood imbalance}
Neighborhood imbalance refers to nodes of some classes having a greater proportion of heterogeneous nodes in their neighborhood compared to other classes. Fig.~11 demonstrates that these imbalances can exist even when classes have equal node proportions. In the red class, the neighborhoods of three nodes respectively contained $\frac{1}{2}$, $\frac{1}{3}$, and $\frac{1}{3}$ nodes from other classes. However, in the blue class, the neighborhoods contained 1, 1, and $\frac{1}{2}$ nodes respectively.

The imbalance of the neighborhood directly impairs node feature encoding in classes with a higher proportion of heterogeneous neighbors. The reason for this is that DNNs used in graph node classification tasks typically adopt message passing mechanisms. These mechanisms exchange feature information between adjacent nodes layer by layer. Nodes with a large proportion of heterogeneous neighbors are susceptible to negative influence in feature encoding and the final prediction.

\begin{figure}[h] %
\centering
    \includegraphics[width=0.5\linewidth]{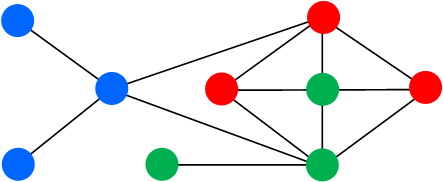} 
     \vspace{-0.05in} 
    \caption{Graph with three-category nodes.}
    \label{fig3}
\end{figure}

\subsubsection{Local imbalance}
As previously described, global imbalance refers to that the proportion/variance/distance/neighborhood/quality imbalance occurs between/among classes. On the contrary, local imbalance is related to the local areas of a or several classes. Due to the complexity of data distributions, only several typical examples of local imbalance are referred to in this study.

Fig.~4(a) presents an example of local imbalance, where the blue class contains two regions. In this scenario, unfairness may occur between the two local regions. Let $\alpha$ represent the proportion of one region, and thus that of the other region is $1-\alpha$. We show that as $\alpha$ decreases, the degree of imbalance between the two regions increases. The data in class `+1' still follow $\mathcal{N}\left([-4, -4]^T, \sigma^{2}\right)$, while the two sub-areas of data in class `-1' follow $\mathcal{N}\left([1,1]^T, \sigma^{2}\right)$ and $\mathcal{N}\left([3,3]^T, \sigma^{2}\right)$, respectively. It is easy to prove that the performance gap of the two areas with the optimal linear classifier becomes large with the decrease of the $\alpha$ value. A smaller value of $\alpha$ results in a larger performance gap and thus a higher degree of imbalance. This type of local imbalance is actually the intra-class imbalance investigated in Refs.~\cite{Tang2022}\cite{Liu2021}. In Ref.~\cite{Tang2022}, areas are divided according to attributes; in Ref.~\cite{Liu2021}, areas are divided according to learning difficulties. Indeed, learning difficulty can also be viewed as an intrinsic training property of data. Fig.~12 presents two additional examples. In the left figure, the proportions of the four regions are equal. However, in the upper regions, the blue class is dominant, while in the lower regions, the yellow class is dominant. In the right figure, the covariance matrices of the four regions are identical. Nevertheless, in the upper regions, the yellow class is dominant, whereas in the lower regions, the blue class is dominant.


\begin{figure}[t] %
\centering
    \includegraphics[width=0.999\linewidth]{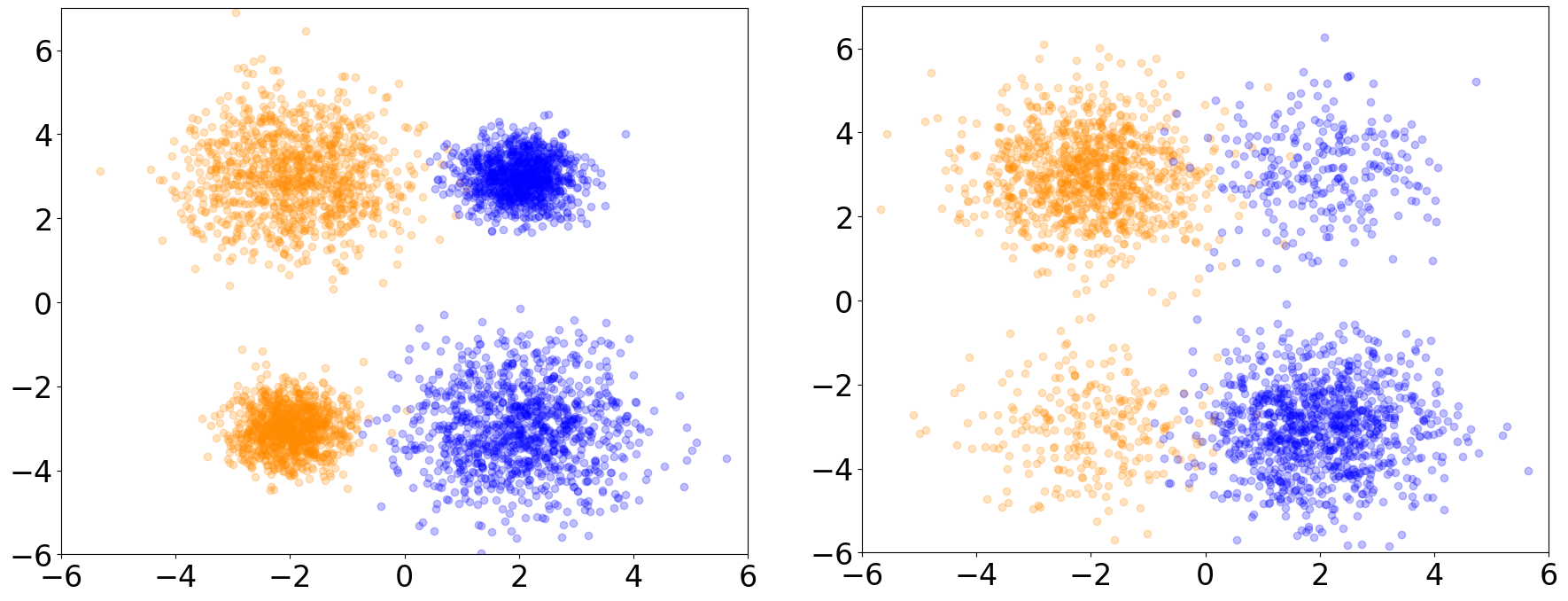} 
     \vspace{-0.2in} 
    \caption{Two examples of local imbalance. Left: Both the yellow and the blue classes contain two regions. Their class proportions are equal. Nevertheless, local variance imbalance exists in both the up and the down areas. Right: The variances of all regions are identical. Nevertheless, local proportion imbalance exists in both the up and the down areas.}
    \label{fig3}
\end{figure}


Due to local imbalances, some methods that aimed to eliminate global imbalances are no longer suitable for use. Global imbalance learning methods adopt the same strategy for each sample in a class. However, different subareas of local imbalance require tailored learning strategies as they contribute to the imbalance differently. For example, in Fig.~12, achieving a fair model requires the use of different, tailored learning strategies in different subareas.

\subsubsection{Mixed imbalance}
It is unlikely that only one type of imbalance among proportion, variance, distance, and quality occurs in real learning tasks. This is because it is impossible to guarantee that the variances, distances, and qualities of different classes are the same. Theoretically, any combination of two or more types of imbalance, including both levels, can occur simultaneously. Because of space limitations, this subsection analyzes only one common case of mixed imbalance, with both proportion and variance imbalances.
Let $\boldsymbol{\theta}=({\eta, \ldots, \eta})^T \in R^d$. Considering a binary learning task in which the data follow 
\begin{equation}
\begin{aligned}
&\operatorname{Pr}(y=+1) = p_{+}, \quad \operatorname{Pr}(y=-1) = p_{-}, \\&\boldsymbol{x} \sim\left\{\begin{array}{ll}\mathcal{N}\left(\boldsymbol{\theta}, \sigma_1^{2} I\right), & \text { if } y=+1, \\ \mathcal{N}\left(-\boldsymbol{\theta}, \sigma_2^{2} I\right), & \text { if } y=-1.\end{array}\right.
\end{aligned}
\label{imbfenbu}
\end{equation}
where $p_{+}:p_{-} =1:V \quad (V>1)$ and $\sigma_1^{2} : \sigma_2^{2} = 1:K \quad (K>1)$. There are two types of imbalances in this learning task: proportion and variance. Regarding proportion imbalance, the proportion of class `-1' is greater than that of class `+1'. As for variance imbalance, the variance factor of class `+1' is less than that of class `-1'. To determine the predominant class, we first prove the following theorem.
\begin{thm}
For the abovementioned binary classification task, the optimal linear classifier $f_{\text{opt}}$ that minimizes the average classification error is 
\begin{equation}
   f_{\text{opt}}=\arg\underset{f}{ \min } \operatorname{Pr}(f(x) \neq y). \label{optf}
\end{equation}
It has the intra-class standard error for the two classes:
\begin{equation}
\small
\begin{aligned} & \mathcal{E}\left(f_{\text{opt}},+1\right)
\\&=\operatorname{Pr}\left\{\mathcal{N}(0,1)<-K\sqrt{B^2+q(K,V)}-B)\right\}, \\ & \mathcal{E}\left(f_{\text{opt}},-1\right) \\&= \operatorname{Pr}\left\{\mathcal{N}(0,1)<KB+\sqrt{B^2+q(K,V)}\right\}, \end{aligned}
\end{equation}
where $B= \frac{-2d\eta}{\sqrt{d}\sigma(K^2-1)}$ and $q(K,V)=\frac{2\text{log}(\frac{K}{V})}{K^2-1}$.\label{Bal_thm1}
\end{thm}
\begin{proof} With the similar inference manner used in Ref.~\cite{Xu2021}, it is easy to obtain that $f_{opt}(x) = x + b$ (that is, $w = \textbf{1}$). Then the generalization error of $f_{opt}(x)$ is
\begin{equation}
\small
\begin{aligned} 
\mathcal{E}\left(f_{\text{opt}}\right) &=V\cdot\operatorname{Pr}\left\{\sum_{i=1}^{d}x_{i}+b>0 \mid y=-1\right\}
\\&+\operatorname{Pr}\left\{\sum_{i=1}^{d}x_{i}+b<0 \mid y=+1\right\}
\\&=V\cdot\operatorname{Pr}\left\{\mathcal{N}(0,1)<\frac{1}{K}(-\frac{\sqrt{d}\eta}{\sigma}+\frac{b}{\sqrt{d} \sigma})\right\}
\\&+\operatorname{Pr}\left\{\mathcal{N}(0,1)<-(\frac{\sqrt{d}\eta}{\sigma}+\frac{b}{\sqrt{d} \sigma})\right\}.
\end{aligned}
\end{equation}
The optimal $b^*$ to minimize $\mathcal{E}\left(f_{\text{opt}}\right)$ is achieved at the point that $\frac{\partial \mathcal{E}\left(f_{\text{opt}}\right)}{\partial b} =0$. Then we can get the optimal $b^*$:
\begin{equation}
    \begin{aligned}
    &b^* = -\frac{d\eta(K^2+1)}{K^2-1}+K\sqrt{4d^2{\eta}^2+2d(K^2-1)\sigma^2\text{log}(\frac{K}{V})}).
    \end{aligned}\label{optimal_brob1}
\end{equation}
Plugging (22) into (21), the generalization errors under the optimal linear classifier for the two classes can be obtained as follows:
\begin{equation}
\small
\begin{aligned} & \mathcal{E}\left(f_{\text{opt}},+1\right) =\operatorname{Pr}\left\{\mathcal{N}(0,1)<-(\frac{\sqrt{d}\eta}{\sigma}+\frac{b^*}{\sqrt{d} \sigma})\right\}
\\&=\operatorname{Pr}\left\{\mathcal{N}(0,1)<-K\sqrt{B^2+q(K,V)}-B)\right\}, \\ & \mathcal{E}\left(f_{\text{opt}},-1\right) =\operatorname{Pr}\left\{\mathcal{N}(0,1)<\frac{1}{K}(-\frac{\sqrt{d}\eta}{\sigma}+\frac{b^*}{\sqrt{d} \sigma})\right\}\\&= \operatorname{Pr}.\left\{\mathcal{N}(0,1)<KB+\sqrt{B^2+q(K,V)}\right\}, \end{aligned}\label{bal-Rnat1}
\end{equation}
where $B= \frac{-2d\eta}{\sqrt{d}\sigma(K^2-1)}$ and $q(K,V)=\frac{2\text{log}(\frac{K}{V})}{K^2-1}$.
\end{proof}

Both $K$ and $V$ influence the performance according to Theorem \ref{Bal_thm1}. We then show how the classification errors of the two classes change as the variations of $K$ and $V$ .

\begin{corollary}
For the learning task investigated in Theorem 3,
\begin{itemize}
    \item if $K \equiv V$, then $\mathcal{E}\left(f_{\text{opt}},+1\right) \equiv \mathcal{E}\left(f_{\text{opt}},-1\right)$;
    \item if $K$ is fixed, then when $V < K$, the performance gap will be decreased with the increasing of $V$;  when $V > K$, the performance gap will be increased with the increasing of $V$;
    \item if $V$ is fixed, then when $K > V$, the performance gap will be increased with the increasing of $K$; when $K < V$, the performance gap will be increased at first and then decreased with the increasing of $K$
    
\end{itemize}  
 \label{bal_coro1}
\end{corollary}
\begin{figure}[t]      \centering
\includegraphics[width=0.7\linewidth]{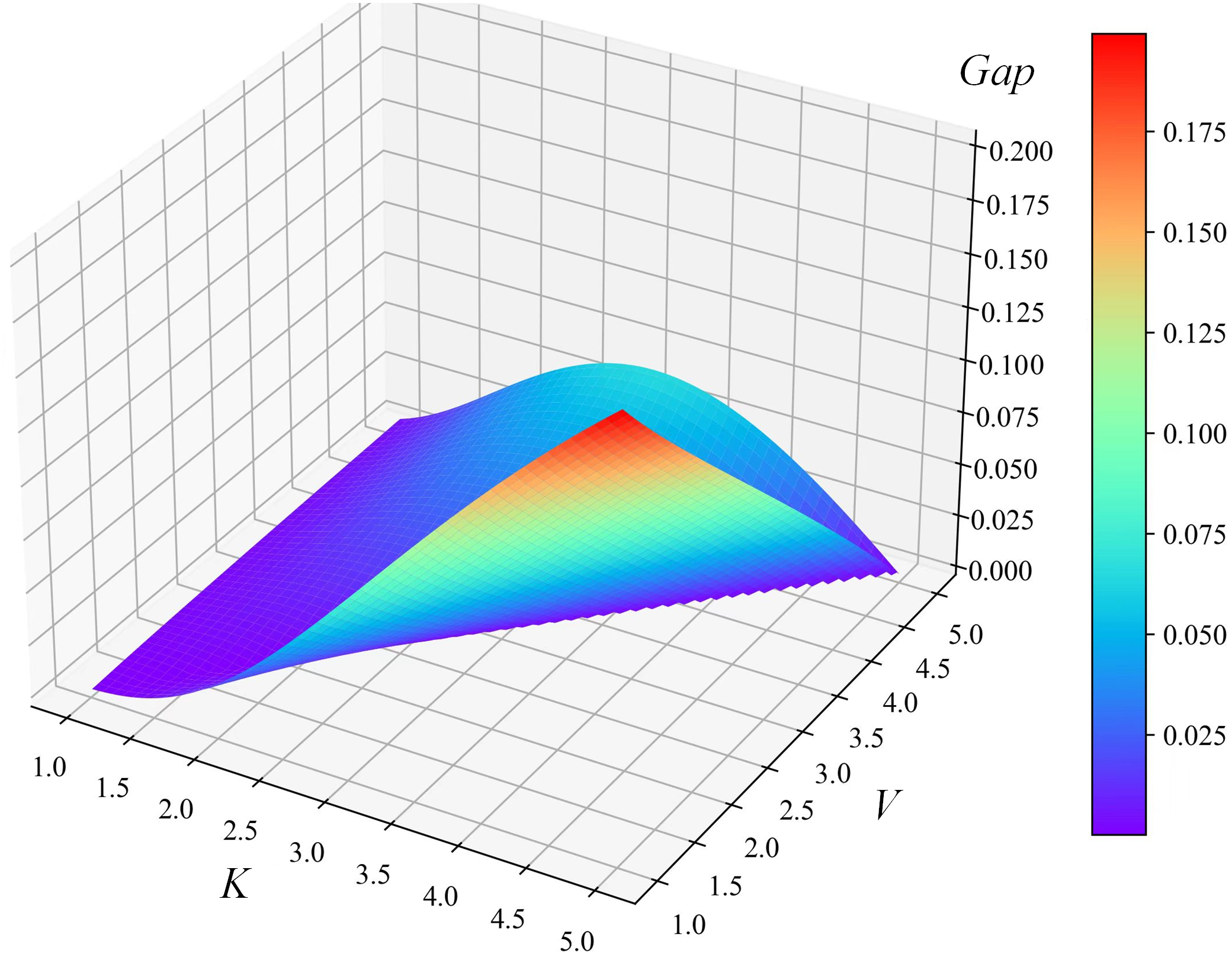} 
     \vspace{-0.2in} 
    \caption{An illustration of the gap under different $K$ and $V$ values.}
    \label{fig3}
\end{figure}
The first conclusion can be directly obtained as $log(\frac{K}{V}) = 0$. The second and the third conclusions can be proved by analyzing the variation of $q$ in (20). According to Corollary 1, when two types of imbalances exist, the type that has a larger imbalance degree will determine the unfairness, or performance gap, between the classes. Fig.~13 illustrates the gap between the two classes under different $K$ and $V$ values when other distribution parameters are set. The largest gap appears in the case that $V \approx 1$ and $K = 5$, indicating that variance imbalance may have a more negative influence on fairness than proportion imbalance in certain situations.

A number of studies on imbalance learning indicate that the simple reweighting or resampling strategies based on class proportions are ineffective in real-world data sets. For example, Megahed et al.~\cite{Megahed} concluded that re-sampling is useful to deal with class imbalance, whereas Goorbergh et al.~\cite{Goorbergh} held the oppose perspective. Corollary 1 may provide a possible theoretical explanation that even if proportion imbalance exists, when there is variance imbalance and $V > K$, increasing the weight of the class `+1' will increase the performance gap and lead to greater unfairness. In practice, proportion imbalance is easy to observe. However, other types of imbalances are often ignored, which may cause algorithm designers to focus primarily on proportion imbalance. This ignorance can result in the ineffectiveness of designed imbalance learning algorithms.

\begin{figure}[t] 
    \includegraphics[width=1\linewidth]{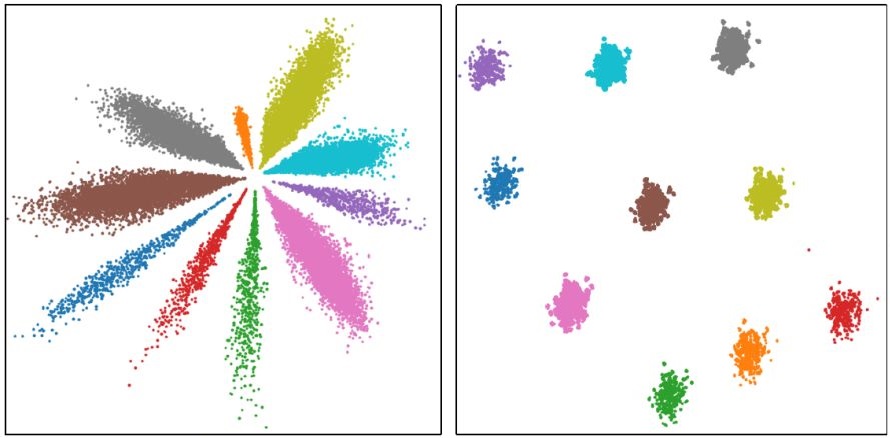} 
     \vspace{-0.25in} 
    \caption{The left figure shows the feature distribution from conventional DNNs; the right one shows the improved feature distribution with regularization on inter-class distance~\cite{Hayat2019}.}
    \label{fig3}
\end{figure}

Fig.~14 is directly borrowed from \cite{Hayat2019}. The left figure shows imbalances in proportion, variance, and distance. In the improved feature distribution of the right figure, the variance imbalance seems to have disappeared, and the distance imbalance has been significantly alleviated as all inter-class distances have been considerably increased.

\section{A new imbalance learning method}
 As multiple types of imbalances are inevitable in real-world applications, it is desirable to investigate effective methods that can address more than one type of imbalance. This subsection initially analyzes two classical methods that have been recently proposed. Thereafter, our method is proposed.

\subsection{Discussion on Logit Adjustment}
Logit adjustment (LA) is a simple yet quite effective imbalance learning method. It adjusts the logits and yields the following loss: 
\begin{equation} 
\begin{aligned}
l(x,y)&=-log\frac{e^{w_y^Ta+b_y+\lambda log \pi_y}}{e^{w_y^Ta+b_y+\lambda log \pi_y}+\sum_{c\neq y}e^{w_{c}^Ta+b_c+\lambda log \pi_c}}
\end{aligned},
\end{equation}
which exerts larger margins to tail classes. The theoretical basis of LA is the following two assumptions: 
\begin{equation} 
\begin{aligned}
& p(y|a) \propto e^{w_y^Ta+b_y} \\
&  p^{bal}(y|a) \propto  p(y|a)/ p(y)
\end{aligned}.
\end{equation}
To derive the first assumption, we rely on the employed softmax loss. We hypothesize that this assumption is predicated on the assumption that the feature co-variance matrices of each class are equal. For binary classification tasks, this means that the feature co-variance matrices are identical. The conditional probability density functions ($p(a|y)$) for two classes are $\mathcal{N}\left(a|\boldsymbol{\mu}_1, \Sigma_1\right)$ and $\mathcal{N}\left(a|\boldsymbol{\mu}_2, \Sigma_2\right)$, respectively. Then we have 
\begin{equation} 
\begin{aligned}
& p(y_1|a) =\frac{\mathcal{N}\left(a|\boldsymbol{\mu}_1, \Sigma_1\right)p(y_1)}{\mathcal{N}\left(a|\boldsymbol{\mu}_1, \Sigma_1\right)p(y_1) + \mathcal{N}\left(a|\boldsymbol{\mu}_2, \Sigma_2\right)p(y_2)}  \\
&  p(y_2|a) =\frac{\mathcal{N}\left(\boldsymbol{\mu}_2, \Sigma_2\right)p(y_2)}{\mathcal{N}\left(a|\boldsymbol{\mu}_1, \Sigma_1\right)p(y_1) + \mathcal{N}\left(a|\boldsymbol{\mu}_2, \Sigma_2\right)p(y_2)}
\end{aligned}.
\end{equation}
When $\Sigma_1 \equiv \Sigma_2 = \Sigma$, Eq.~(26) becomes
\begin{equation} 
\begin{aligned}
& p(y_1|a) \propto e^{\mu_1^T\Sigma^{-1}a-\frac{1}{2}\mu_1^T\Sigma^{-1}\mu_1+logp(y_1)}  \\
&  p(y_2|a) \propto e^{\mu_2^T\Sigma^{-1}a-\frac{1}{2}\mu_2^T\Sigma^{-1}\mu_2+logp(y_2)}
\end{aligned},
\end{equation}
which are in accordance with the first assumption in Eq.~(25)\footnote{If $\Sigma_1 \neq \Sigma_2$, then the coefficient for $x^Tx$ is not zero. The optimal classification boundary is thus not linear}. The inference above implies that LA works on classes with equal co-variance. Moreover, LA also ignores distance similarity and assumes that the data are clean. When simultaneously existing with any of the other three types of imbalances, LA will be adversely affected and become less effective.

\begin{figure}[t] %
\centering
    \includegraphics[width=0.6\linewidth]{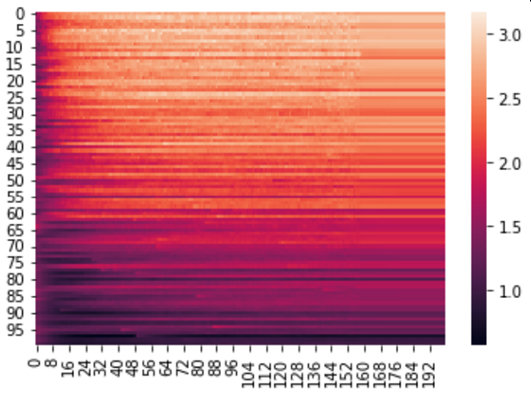} 
     \vspace{-0.15in} 
    \caption{The norm of the weight coefficients for classes from head (`0') to tail (`100') along with the 200 training epochs~(the x-axis).}
    \label{fig3}
     \vspace{-0.13in}
\end{figure}

\begin{figure}[t] %
\centering
    \includegraphics[width=0.86\linewidth]{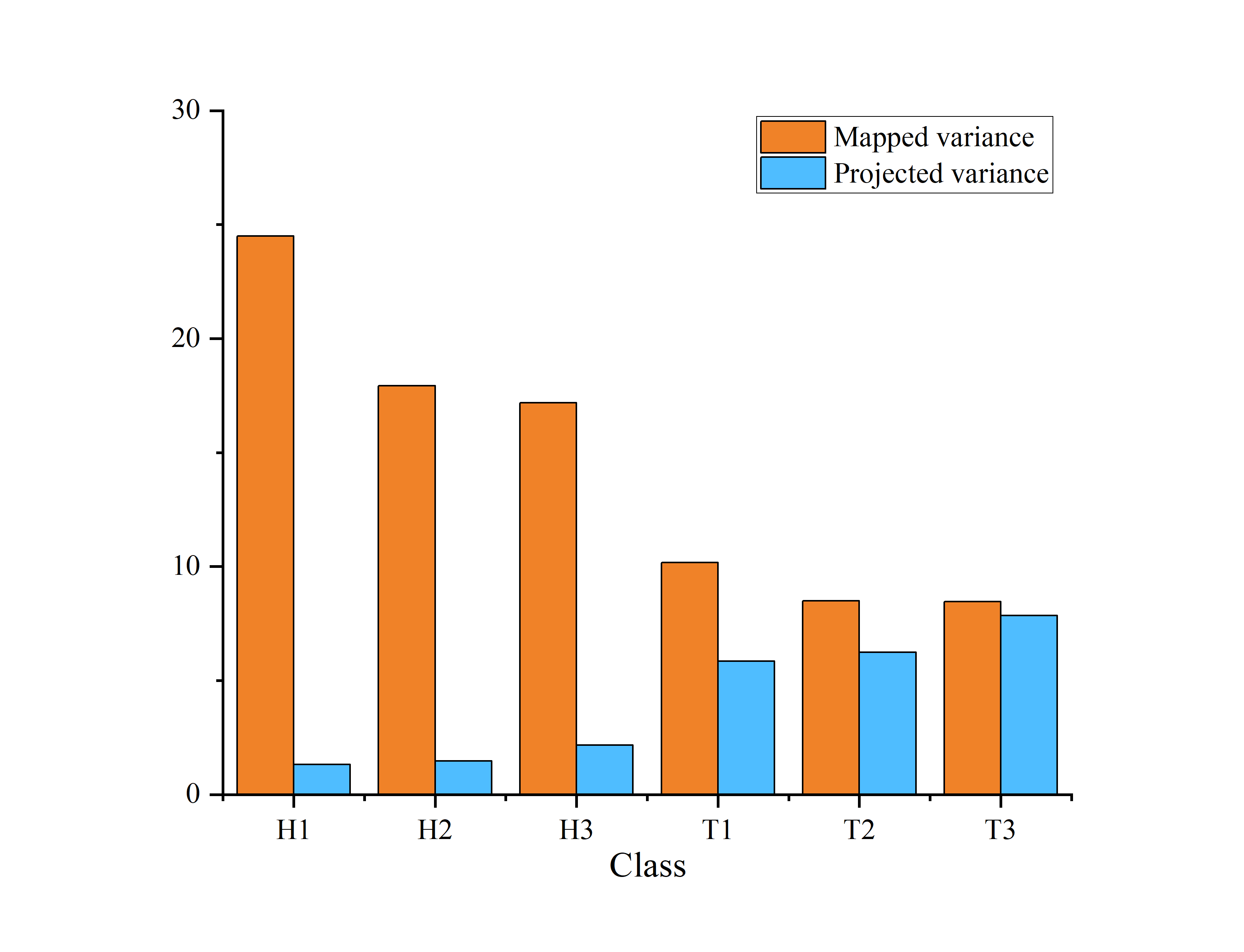} 
     \vspace{-0.35in} 
    \caption{The projected and mapped class variances of the three head and three tail categories.}
    \label{fig3}
    \vspace{-0.1in}
\end{figure}

\subsection{Discussion on ISDA}

In Section III-A, Eq.~(6) depicts the ISDA loss. Our analysis reveals that ISDA deals with variance imbalance in essential. It is failed in imbalance data corpora~\cite{LMY2022}. This subsection discusses why ISDA is unsuitable for imbalance learning in more detail.

In ISDA, the mapped class variance is used as the perturbation term. Note that $\triangle w_{yc} = w_y - w_c$. This scheme will increase the perturbation of the classes with large norms of $w_y$ in Eq.~(6). In learning on imbalance corpus, the coefficients' norms of the majority classes are usually larger than those of minority classes. Fig.~15 shows the norms of coefficients on a benchmark long-tail corpus. Head classes have larger norms of feature coefficients than the tail classes. Fig.~16 shows the mapped class variance of the three head~(H1, H2, and H3) and the three tail~(T1, T2, and T3) classes on a benchmark imbalance dataset CIFAR10-LT~\cite{menon2020long}. The perturbations using ISDA for the head classes are larger than those for the tail classes. Accordingly, head categories benefit more from logit perturbation using ISDA than tail categories. ISDA will further exacerbate the performances on tail categories.

We argue that the projected class variance rather than the mapped class variance is more appropriate. The projected class variance is defined as follows:
\begin{equation}
\sigma_p=\frac{{w^*}^T\sum_+w^*}{||w^*||_2^2}=\frac{{w^*}^T\sum_+w^*}{{w^*}^Tw^*}. 
\end{equation}           

Fig.~16 also shows the projected class variances on the CIFAR10-LT data set. The projected variances for head classes are smaller than those for tail classes. The ISDA loss in Eq.~(6) becomes
\begin{equation} 
\begin{aligned}
l(x,y)&=-log\frac{e^{w_y^Ta+b_y}}{e^{w_y^Ta+b_y}+\sum_{c\neq y}e^{w_{c}^Ta+b_c+\lambda \frac{{\triangle w_{yc}}^T\Sigma_y \triangle w_{yc}}{{\triangle w_{yc}}^T \triangle w_{yc}}}}
\end{aligned},
\end{equation}
which is called normalized ISDA~(NISDA) in this study. Many tail categories possess high projected class variances. Therefore, the NISDA incurs significant loss increments, which result in the increased margin for these tail categories. 

\subsection{Our Proposed Method}
Our previous analysis has indicated that classical and other existing methods only consider one type of imbalance, whereas, in real learning tasks, various types of imbalances are likely to exist. Owing to space constraint, this paper omits the bias term. A new loss is proposed as follows:
\begin{equation} 
\begin{aligned}
&l(x,y)=
\\&-log\frac{e^{w_y^Ta}}{e^{w_y^Ta}+\underset{{c\neq y}}{\sum}e^{w_{c}^Ta+\lambda_{0} log{\frac{\pi_y}{\pi_c}}+\lambda_{1}\frac{{\triangle w_{yc}}^T\Sigma_y \triangle w_{yc}}{{||\triangle w_{yc}||^2_2 }}+\lambda_{2}log\frac{\bar{\Delta}}{\triangle_{yc}}}}
\end{aligned},
\end{equation}
where $\bar{\Delta}$ is the average distance between centers of each pair of classes, and $\triangle_{yc}$ is the distance between centers of classes $y$ and $c$; $\lambda_0$, $\lambda_1$, and $\lambda_2$ are three hyper-parameters.

If there is no proportion and variance imbalances, then $\pi_1 = \cdots = \pi_C = \frac{1}{C}$ and the values of $\frac{{\triangle w_{yc}}^T\Sigma_y \triangle w_{yc}}{{||\triangle w_{yc}||^2_2}}$ are equal for each $y$ and $c$. Eq.~(30) is reduced to
\begin{equation} 
\begin{aligned}
l(x,y)=-log\frac{e^{w_y^Ta}}{e^{w_y^Ta}+{\sum_{c\neq y}}e^{w_{c}^Ta+\lambda_2log\frac{\bar{\Delta}}{\triangle_{yc}}}}
\end{aligned},
\end{equation}
which exerts larger perturbations on classes with smaller distances to other classes. This is reasonable and can deal with distance imbalance.

If there is no proportion and distance imbalances, then $\pi_1 = \cdots = \pi_C = \frac{1}{C}$ and the values of ${||\triangle w_{yc}||_2}$ are equal for each $y$ and $c$. Eq.~(30) is reduced to 
\begin{equation} 
\begin{aligned}
l(x,y)=-log\frac{e^{w_y^Ta}}{e^{w_y^Ta}+{\sum_{c\neq y}}e^{w_{c}^Ta+\lambda'_1{\triangle w_{yc}}^T\Sigma_y \triangle w_{yc}}}
\end{aligned},
\end{equation}
which is the ISDA loss.

Similarly, if there are no variance and distance imbalances, Eq.~(30) simplifies to the LA loss along with a constant perturbation value. It is essential to note that the proposed method solely addresses global imbalance.

There are three hyper-parameters in our proposed loss. It is challenging to select an appropriate hyper-parameter setting. In this study, meta learning~\cite{Shu1917} is used and more hyper-parameters are introduced with the following loss:
\begin{equation} 
\footnotesize
\begin{aligned}
&l(x,y)=
\\&-log\frac{e^{w_y^Ta}}{e^{w_y^Ta}+\underset{{c\neq y}}{\sum}e^{w_{c}^Ta+\lambda_{yc_1} log{\frac{\pi_y}{\pi_c}}+\lambda_{yc_2}\frac{{\triangle w_{yc}}^T\Sigma_y \triangle w_{yc}}{{||\triangle w_{yc}||^2_2 }}+\lambda_{yc_3}log\frac{\bar{\Delta}}{\triangle_{yc}}}}
\end{aligned},
\end{equation}
where $\lambda_{yc_1}$, $\lambda_{yc_2}$, and $\lambda_{yc3}$ are the newly introduced hyper-parameters for the class $y$. Compared with the loss in Eq.~(30), the number of hyper-parameters becomes $3(C-1)$ in the loss of Eq.~(33).

Assuming that we have a small amount of balanced meta data $D^{meta} = \{x^{mt}_i,y^{mt}_i\}$, $i=1,\cdots,M$~($M << N$). Let $\Theta$ be the parameters of the backbone network, and $\Omega$~(=\{$\lambda_{yc_1}$, $\lambda_{yc_2}$, $\lambda_{yc_3}$\}, $y,c \in \{1,\cdots C\}$ and $y\neq c\}$) be the hyper-parameters in Eq.~(33). Given a batch of training samples $\{x_i,y_i\}$, $i=1,\cdots,n$ and a batch of meta samples $\{x_j,y_j\}$, $j=1,\cdots, m$. The training with meta learning consists of three main steps. 

First, a temporary update for $\Theta$ is conducted as follows:
\begin{equation}
    \hat{\Theta}^{t}(\Omega) = {\Theta}^{t}-\eta_1 \frac{1}{n} \sum_{i=1}^n \nabla_{\Theta} l_{\Omega}(f_{\Theta}(x_i),y_i)|_{{\Theta}^{t}},
\end{equation}
where $\eta_1$ is the step size, $l_{\Omega}$ is actually the loss defined in Eq.~(32), and $f_{\Theta}$ is the backbone network. Secondly, $\Omega$ is updated on a batch of $m$ meta data.
\begin{equation}
    {\Omega}^{t+1} = {\Omega}^{t}-\eta_2 \frac{1}{m} \sum_{j=1}^m \nabla_{\Omega} l_{\Omega}(f_{\hat{\Theta}^{t}}(x_j),y_j)|_{\Omega^{t}},
\end{equation}
where $\eta_2$ is the step size. Finally, the update for $\Theta$ is conducted as follows:
\begin{equation}
   {\Theta}^{t+1} = {\Theta}^{t}-\eta_1 \frac{1}{n} \sum_{i=1}^n \nabla_{\Theta} l_{\Omega^{t+1}}(f_{\Theta}(x_i),y_i)|_{{\Theta}^{t}}.
\end{equation}

During the training process, these three steps are performed repeatedly. Our method is called \textbf{meta} \textbf{l}ogit \textbf{ad}justment~(MetaLAD) for briefly. The whole algorithmic steps are shown in Algorithm 1. Since the calculation for the $\triangle w_{yc}$, $\triangle_{yc}$, and $\bar{\triangle}$ has relatively low time complexity, the computational complexity of our method is comparable to that of MetaSAug.

\begin{algorithm}[t]
\small
    \caption{MetaLAD}
    \label{alg1}
    \textbf{Input}: $D^{\text{train}}$, $D^{\text{meta}}$, step sizes $\eta_{1}$ and $\eta_{2}$, batch size $n$, meta batch size $m$, ending steps $T_1$ and $T_2$.
       \textbf{Output}: Trained network $f_{_{\boldsymbol{\Theta}}}$.
    \begin{algorithmic}[1]\vspace{-0.0in}
\STATE {Initialize $\Omega$ and networks $f_{_{\boldsymbol{\Theta}}}$;}
\FOR{$t = 1$ to $T_1$}
    \STATE{Sample $n$ samples from $D^\text{train}$;}
    \STATE{Calculate the standard CE loss on these samples;}
    \STATE{Update $\Theta$ using SGD;}
\ENDFOR
\FOR {$t = T_1+1$ to $T_2$}
    \STATE{Sample $n$~(denoted as $B_n$) and $m$~(denoted as $B_m$) samples from $D^\text{train}$ and $D^\text{meta}$, respectively;}
    \STATE{Obtain current covariance matrices $\Sigma_c$ for each class;}
    \STATE{Calculate $\triangle w_{yc}$ for all classes;}
    \STATE{Calculate $\bar{\triangle}$ and $\triangle_{yc}$ for all classes;}
    \STATE{Calculate the loss on $B_n$ based on Eq.~(33); }
    \STATE{Calculate $\hat{\Theta}^{t}(\Omega)$ using Eq.~(34);}
    \STATE{Calculate the loss on $B_m$ based on Eq.~(33); }
    \STATE{Update $\Omega$ using Eq.~(35);}
    \STATE{Calculate the new batch loss on $B_n$ based on Eq.~(33) \\ with updated $\Omega$; }
    \STATE{Update $\Theta$ using Eq.~(36);}
\ENDFOR
    \end{algorithmic}
\end{algorithm}

\section{Experiments}
Experiments are conducted to evaluate the proposed method MetaLAD on two typical scenarios including training on standard datasets and imbalance datasets.

\subsection{Experiments on Standard Datasets}
Two benchmark datasets are involved in this part including CIFAR10 and CIFAR100~\cite{Krizhevsky32}. In both datasets, there are 50,000 images for training and 10,000
images for testing. The training and testing configurations utilized in~\cite{LMY2022} are adopted. 

Several classical and state-of-the-art
robust loss functions and logit perturbation methods are compared: Large-margin loss~\cite{liu2016large}, Disturb label~\cite{xie2016disturblabel}, Focal Loss~\cite{lin2017focal}, Center loss~\cite{wen2016discriminative}, Lq loss~\cite{zhang2018generalized}, ISDA, ISDA + Dropout, MetaSAug~\cite{li2021metasaug}, and LPL~\cite{LMY2022}. Wide-ResNet-28-10 (WRN-28-10)~\cite{Zagoruyko29} and ResNet-110~\cite{KaimingHe30} are used as the base neural networks. The results reported in the LPL paper for the above competing methods, some of which are from the original papers of the individual algorithms, are presented directly as the training/testing configuration is identical for both sets. The training settings for the base neural networks mentioned above follow the instructions given in the ISDA paper and its released codes.

The hyper-parameters for our method MetaLAD are set according to Shu et al.~\cite{Shu1917}. A meta set is constructed for each dataset by randomly selecting ten images per class from the training set. The top-1 error is leveraged as the evaluation metric. The base neural networks are re-run with the original cross-entropy~(CE) loss to ensure a fair comparison of performance. Stochastic gradient descent~(SGD) is used over a total of 240 epochs. The initial learning rate~($\eta_1$) was set to 0.1, and we applied learning rate decay at the 160th and 200th epochs with a decay coefficient of 0.1. The momentum was set to 0.9, and the weight decay to 5e-4. The parameters $\Omega$ for the meta-learning module are initialized to a value of \{1, 1, 1\} for each class. Since the gradient for $\Omega$ is typically small during updates, we set the learning rate~($\eta_2$) for the meta-learning part to a higher value of 1e2. $T_1$ is set as 160, and the other settings follow the ones in MetaSAug.

Tables I and II show the top-1 errors of the competing methods on the two balanced datasets CIFAR10 and CIFAR100. 
MetaLAD achieves the lowest top-1 errors on both datasets under two different backbone networks. Note that although MetaLAD is based on meta data, these meta data are selected from the training set. Thus, MetaLAD does not utilize any additional data. The comparison suggest that our method is effective for benchmark datasets that are not considered as imbalance.

\begin{table}[tb]
\caption{Mean values and standard deviations of the test Top-1 errors for all the involved methods on CIFAR10.}\label{tab:banlance_cifar10}
\centering
\vspace{-0.08in}
\begin{tabular}{|p{2.68cm}||c||c|}\hline
    Method &  WRN-28-10 & ResNet-110  \\\hline
    Basic & 3.82 ± 0.15\% & 6.76 ± 0.34\% \\
    Large Margin & 3.69 ± 0.10\% & 6.46 ± 0.20\% \\
    Disturb Label & 3.91 ± 0.10\% & 6.61 ± 0.04\% \\
    Focal Loss & 3.62 ± 0.07\% & 6.68 ± 0.22\% \\
    Center Loss & 3.76 ± 0.05\% & 6.38 ± 0.20\% \\
    Lq Loss & 3.78 ± 0.08\% & 6.69 ± 0.07\% \\
    ISDA & 3.60 ± 0.23\% & 6.33 ± 0.19\% \\
    ISDA + Dropout & 3.58 ± 0.15\% & 5.98 ± 0.20\% \\
    MetaSAug &  3.85 ± 0.33\% & 7.22 ± 0.34\%  \\  
    LPL  & 3.37 ±  0.04\% & 5.72 ± 0.05\%  \\ \hline

    MetaLAD  & \textbf{ 2.56 ±  0.15\% }  & \textbf{ 5.04 ± 0.09\% }  \\ \hline
\end{tabular} \vspace{-0.1in}
\end{table}

\begin{table}[tb]
\caption{Mean values and standard deviations of the test Top-1 errors for all the involved methods on CIFAR100.}\label{tab:banlance_cifar100}
\centering
\vspace{-0.08in}
\begin{tabular}{|p{2.68cm}||c||c|}\hline
    Method & WRN-28-10 & ResNet-110  \\\hline
    Basic & 18.53 ± 0.07\% & 28.67 ± 0.44\% \\
    Large Margin & 18.48 ± 0.05\% & 28.00 ± 0.09\% \\
    Disturb Label & 18.56 ± 0.22\% & 28.46 ± 0.32\% \\
    Focal Loss & 18.22 ± 0.08\% & 28.28 ± 0.32\% \\
    Center Loss & 18.50 ± 0.25\% & 27.85 ± 0.10\% \\
    Lq Loss & 18.43 ± 0.37\% & 28.78 ± 0.35\% \\
    ISDA & 18.12 ± 0.20\% & 27.57 ± 0.46\% \\
    ISDA + Dropout & 17.98 ± 0.15\% & 26.35 ± 0.30\% \\
    MetaSAug &  18.61 ± 0.29\% & 28.75 ± 0.22\%  \\  
    LPL  & 17.61 ± 0.30\% & 25.42 ± 0.07\% \\ \hline
    MetaLAD  & \textbf{ 16.49 ± 0.17\% }  & \textbf{ 24.52 ± 0.22\% }  \\ \hline
\end{tabular} 

\label{table_CIFAR100_balance}  \vspace{-0.0in}
\end{table}

\subsection{Experiments on Imbalanced Datasets}
Three benchmark datasets are involved including the imbalance versions of CIFAR10~(i.e., CIFAR10-LT) and CIFAR100~(i.e., CIFAR100-LT) and a large-scale corpus iNaturalist. We followed the training and testing configurations outlined in~\cite{menon2020long}. The iNaturalist corpus comprises two datasets: iNaturalist 2017 (iNat2017)\cite{inat2017} and iNaturalist 2018 (iNat2018)\cite{inat2018}. Both datasets have highly imbalanced class distributions. iNat2017 has 579,184 training images that belong to 5,089 classes, and an imbalance factor of 3919/9. iNat2018 has 435,713 images distributed across 8,142 classes, and an imbalance factor of 500. Several classical and SOTA methods are compared\footnote{Some recent classical methods such as ResLT~\cite{ResLT} and OLTR++~\cite{OLTR} are not involved as these methods are not in the same family of our proposed method. In addition, our method can work together with these methods.}: Class-balanced CE loss, Class-balanced fine-tuning~\cite{Cui4190}, Meta-weight net~\cite{Shu1917}, Focal Loss~\cite{lin2017focal}, Class-balanced focal loss~\cite{cui2019class}, LDAM~\cite{cao2019learning}, LDAM-DRW~\cite{cao2019learning}, ISDA + Dropout, LA, LPL, and KPS~\cite{Li4812}.

In CIFAR10-LT and CIFAR100-LT, Menon et al.~\cite{menon2020long} released the training data under  $\pi_1/\pi_{100} = 100:1$. Therefore, their reported results for some of the above competing methods are directly followed and the training settings are fixed. Similar to the experiments in~\cite{menon2020long}, ResNet-32~\cite{KaimingHe30} is used as the base neural network. The average top-1 error of five repeated runs is presented. 

In comparison to the balanced experiment settings, most hyper-parameters are the same. The difference lies in the meta-learning part, where the learning rate~($\eta_2$) for CIFAR-10-LT is set to 1e2 and for CIFAR-100-LT is set to 1e3.

In iNat2017 and iNat2018, the results of the above competing methods reported in \cite{LMY2022} are directly presented. The results of KPS are from its reported values and released code. Similar to the experiments in~\cite{wu2020dist}, ResNet-50~\cite{KaimingHe30} is used as the base neural network.  The average top-1 error of five repeated runs is presented.
Following Li et al.~\cite{li2021metasaug}, we selected five images per class from the iNat2017 training dataset and two images per class from the iNat2018 training dataset to constitute our meta set. To remain consistent with previous tasks, all hyper-parameters and settings were largely retained, except for the learning rate~($\eta_2$) of the meta-learning component, which was set to 1e3 for both iNat2017 and iNat2018 datasets.

Tables III and IV show the results of all the competing methods on CIFAR10-LT and CIFAR100-LT, respectively.  Our method, MetaLAD, outperforms all other competing methods, including another meta-learning based approach, MetaSAug. However, ISDA exhibits poor results on these two datasets, suggesting that it could increase the disparity between the head and tail categories. This is observed by its inferior performance even to the standard CE loss on CIFAR100-LT.

\begin{table}[tb]
\caption{Test Top-1 errors on CIFAR100-LT (ResNet-32).}\label{tab:longtail_cifar100}
\centering
\vspace{-0.08in}
\begin{tabular}{|p{3.9cm}||c||c|}\hline
    Ratio & 100:1 & 10:1 \\\hline
    Class-balanced CE loss  & 61.23\% & 42.43\% \\
    Class-balanced fine-tuning  & 58.50\% & 42.43\% \\
    Meta-weight net  & 58.39\% & 41.09\% \\
    Focal Loss & 61.59\% & 44.22\% \\
    Class-balanced focal loss  & 60.40\% & 42.01\% \\
    LDAM  & 59.40\% & 42.71\% \\
    LDAM-DRW  & 57.11\% & 41.22\% \\
    ISDA + Dropout & 62.60\% & 44.49\% \\ 
    LA & 56.11\% & 41.66\% \\
    MetaSAug  & 53.13\% & 38.27\% \\
    LPL  & {55.75}\% & {39.03\%} \\KPS  & {54.97}\% & {40.16\%} \\ \hline
    MetaLAD   &  \textbf{51.55}\% &  \textbf{37.45}\% \\ \hline
\end{tabular} \vspace{-0.1in}
\end{table}

\begin{table}[t]
\caption{Test Top-1 errors on CIFAR10-LT (ResNet-32).}\label{tab:longtail_cifar10}
\centering
\vspace{-0.08in}
\begin{tabular}{|p{3.9cm}||c||c|}\hline
    Ratio & 100:1 & 10:1 \\\hline
    Class-balanced CE loss & 27.32\% & 13.10\% \\
    Class-balanced fine-tuning  & 28.66\% & 16.83\% \\
    Meta-weight net & 26.43\% & 12.45\% \\
    Focal Loss & 29.62\% & 13.34\% \\
    Class-balanced focal loss  & 25.43\% & 12.52\% \\
    LDAM  & 26.45\% & 12.68\% \\
    LDAM-DRW  & 25.88\% & 11.63\% \\
    ISDA + Dropout & 26.45\% & 12.98\% \\
    LA & 22.33\% & 11.07\% \\ 
    MetaSAug  & 19.46\% & 10.56\% \\
    LPL  & {22.05}\% & {10.59\%} \\
    KPS  & {18.77}\% & {10.95\%} \\\hline
    MetaLAD & \textbf{17.87}\%   & \textbf{9.63}\%   \\ \hline
\end{tabular} \vspace{-0.15in}

\end{table}

Table V displays the performance results of all competing methods on the iNat2017 and iNat2018 datasets. Similar findings are obtained. MetaLAD achieves the lowest and the second lowest top-1 errors on both datasets. Although MetaLAD's performance on iNat2018 is slightly lower than KPS, MetaLAD has competitive results on common datasets such as CIFAR10 and CIFAR100, whereas KPS is only suitable for (proportion) imbalanced data.

\begin{table}[t]
\caption{Test Top-1 errors on real-world datasets (ResNet-50). }\label{tab:longtail_inat}
\centering
\vspace{-0.08in}
\begin{tabular}{|p{4.2cm}||c||c|}\hline
    Method  & iNat2017 & iNat2018  \\\hline
    Class-balanced CE loss  & 42.02\% & 33.57\% \\
    Class-balanced fine-tuning & 41.77\% &  34.16\% \\
    Meta-weight net  & 37.48\%  & 32.50\% \\
    Focal Loss& 38.98\%  & 72.69\% \\
    Class-balanced focal loss & 41.92\% & 38.88\% \\
    LDAM    & 39.15\% & 34.13\% \\
    LDAM-DRW    &37.84\% & 32.12\%  \\
    ISDA + Dropout    & 43.37\% & 39.92\% \\
    LA    & 36.75\% & 31.56\% \\
    MetaSAug     & 38.47\% & 32.06\% \\
    LPL  & {35.86\%} & {30.59\%} \\
    KPS  & {35.56\%} & {\textbf{29.65}\%} \\\hline
    MetaLAD  &  \textbf{ 35.08\%}& { 29.77\%}   \\ \hline
\end{tabular} \vspace{-0.0in}
\end{table}
\subsection{More Analyses and Discussion}

\subsubsection{Ablation study} 
The proposed loss~(i.e., Eq.~(33)) of our MetaLAD consists of three new items. The first item $log{\frac{\pi_y}{\pi_c}}$ aims to tune the proportion imbalance; the second term ${{\triangle w_{yc}}^T\Sigma_y \triangle w_{yc}}/{{||\triangle w_{yc}||^2_2 }}$ aims to tune the variance imbalance; and the third term $log\frac{\bar{\Delta}}{\triangle_{yc}}$ aims to tune the distance imbalance. To assess the usefulness of each item, we select two datasets for comparison, namely, CIFAR100 and CIFAR100-LT. All the experimental settings follow those used in the above-mentioned experiments. The results, shown in Tables VI and VII, indicate that, except on balanced datasets where the first term is unavailable~($log\frac{\pi_y}{\pi_c} \equiv 0$), each item can result in better performance than the basic method.

\begin{table}[tb]
\caption{Mean values and standard deviations of the test Top-1 errors for Ablation study on CIFAR100.}\label{tab:banlance_cifar10}
\centering
\vspace{-0.08in}
\begin{tabular}{|p{2.68cm}||c||c|}\hline
    Method &  WRN-28-10 & ResNet-110  \\\hline
    Basic  & {18.53 ± 0.07\% }  & { 28.67 ± 0.44\% }  \\ \hline
    First item only  & { /}  & { / }  \\ \hline
    Second term only  & { 17.34 ± 0.13\% }  & { 25.69 ± 0.27\% }  \\ \hline
    Third term only & { 18.04 ± 0.30\% }  & { 25.95 ± 0.15\% }  \\ \hline
    MetaLAD  & { 16.49 ± 0.17\% }  & { 24.52 ± 0.22\% }  \\ \hline
\end{tabular} \vspace{-0.1in}
\end{table}

\begin{table}[tb]
\caption{Mean values and standard deviations of the test Top-1 errors for ablation study on CIFAR100-LT.}\label{tab:banlance_cifar10}
\centering
\vspace{-0.08in}
\begin{tabular}{|p{2.68cm}||c||c|}\hline
    Ratio & 100:1 & 10:1  \\\hline
    Basic  & {61.23}\%   & { 42.43\% }  \\ \hline
    First item only  & {52.06}\% &  {38.01}\% \\ \hline
    Second term only &  {55.65}\% &  {39.27}\% \\ \hline
    Third term only  &  {54.37}\% &  {39.50}\% \\ \hline
    MetaLAD  &  {51.55}\% &  {37.45}\% \\ \hline
\end{tabular} \vspace{-0.1in}
\end{table}

In addition, the second item is modified on the basis of the ISDA loss. Therefore, the comparison between the loss with only the second item and the ISDA is also conducted. The loss with only the second item is called normalized ISDA~(NISDA). NISDA outperforms ISDA~(without dropout) according to the results on the six datasets shown in Fig.~17.

\begin{figure}[t] 
    \centering
    \includegraphics[width=0.8\linewidth]{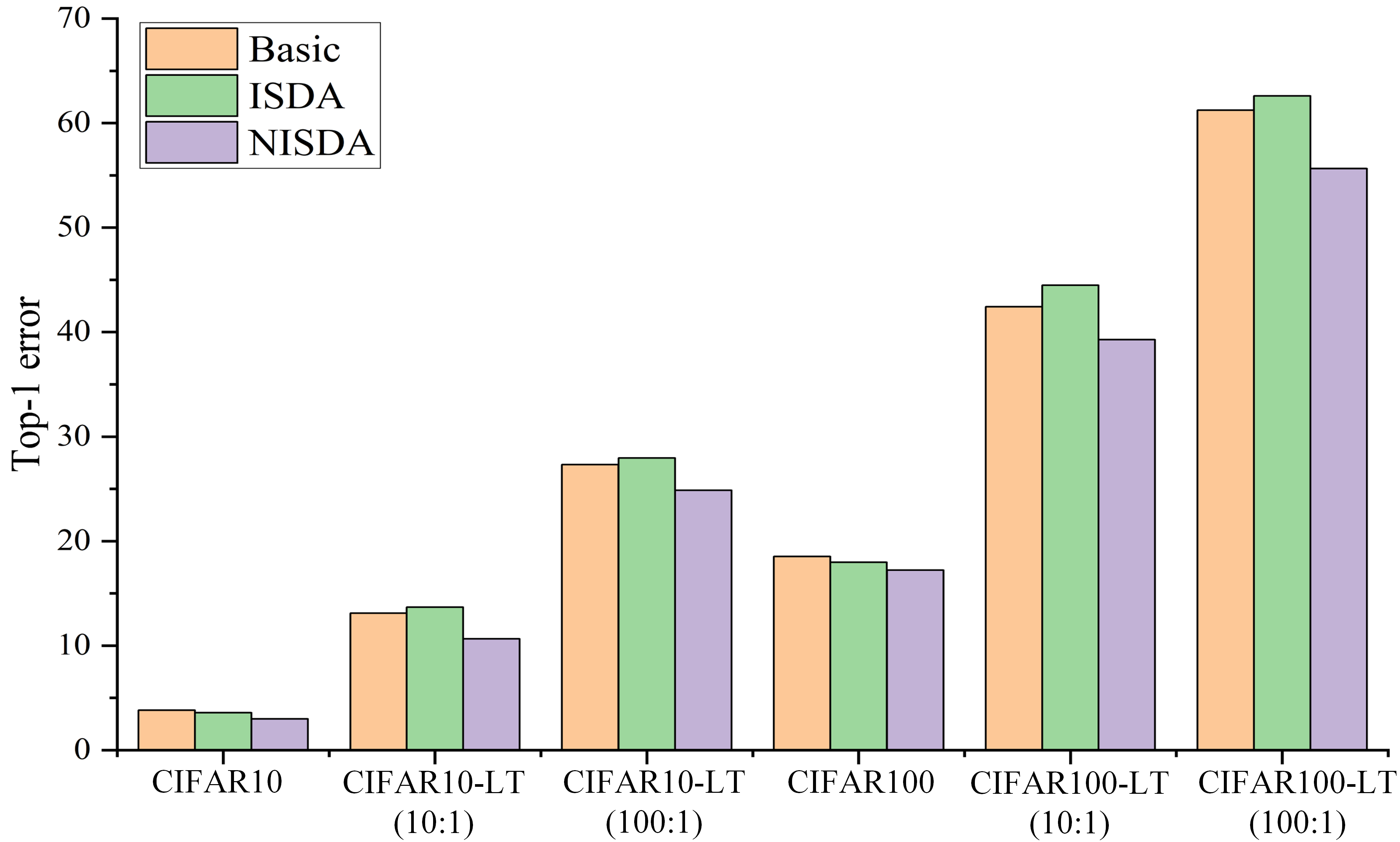} 
     \vspace{-0.15in} 
    \caption{Comparison results of ISDA and NISDA on six datasets. }
    \label{fig3}
     \vspace{-0.1in}
\end{figure}

\subsubsection{Analysis of the perturbation terms}
There are three types of perturbations in our MetaLDA loss.  As the first term $log{\frac{\pi_y}{\pi_c}}$ is constant during training, this part analyzes the rest two terms including the variance and the distance terms. Figs.~18 and 19 show the values of the variance term in three different epochs on CIFAR100-LT~(100:1) and CIFAR100, respectively. Figs.~20 and 21 show the values of the distance term in three different epochs on CIFAR100-LT~(100:1) and CIFAR100, respectively. First, the differences among different classes are significant in terms of both the variance and the distance terms, indicating that both variance and distance imbalances do exist. Second, both terms of all classes tend to be similar when the epoch increases, indicating that both variance and distance imbalances are significantly alleviated during training with our method.   

\begin{figure}[ht]
  \centering
  \includegraphics[width=0.8\linewidth]{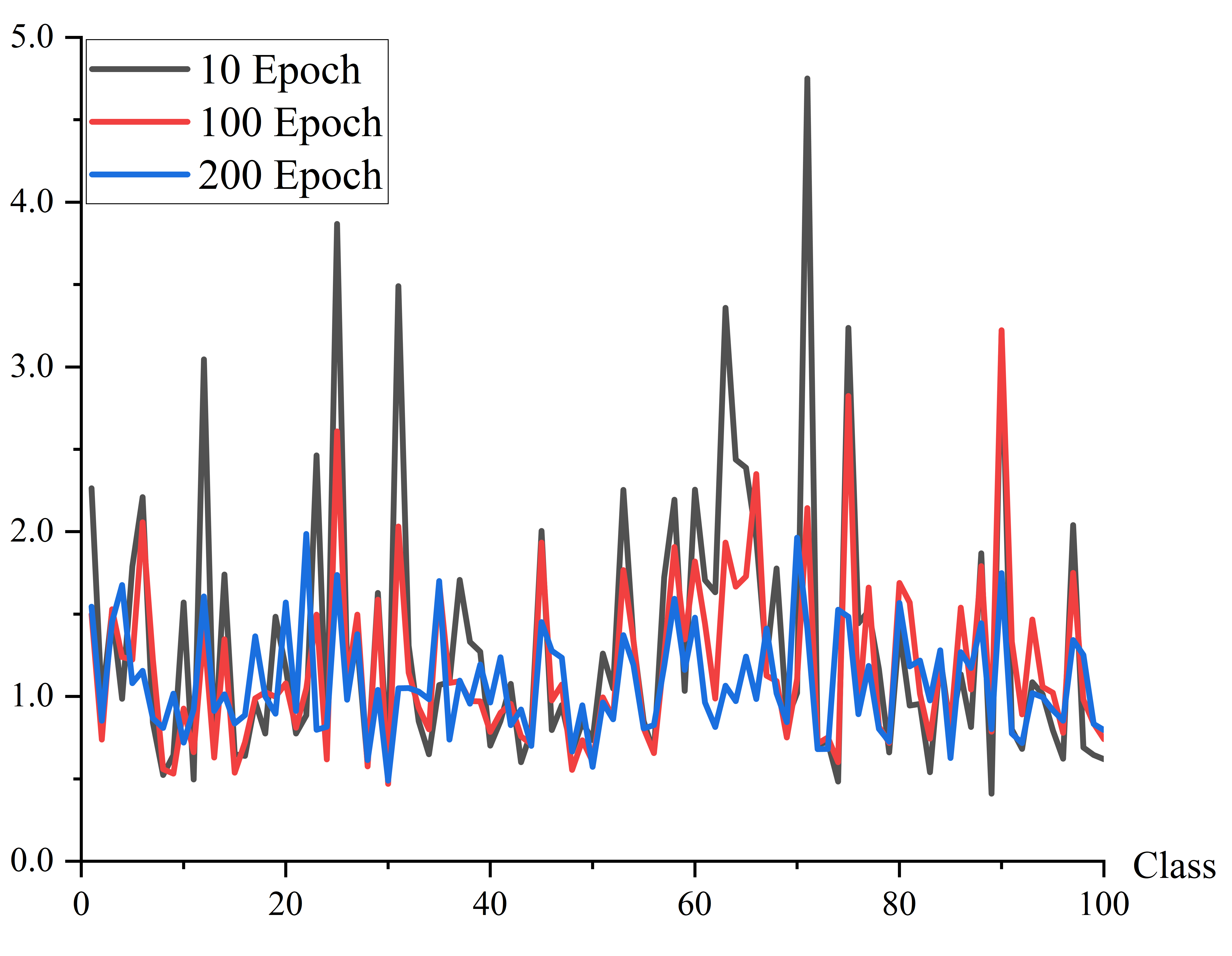}\vspace{-0.15in}
  \caption{The variance terms of different classes at three different epochs on CIFAR100-LT~(100:1).}
  \label{fig:my_label}
\end{figure}

\begin{figure}[ht]
  \centering
  \includegraphics[width=0.8\linewidth]{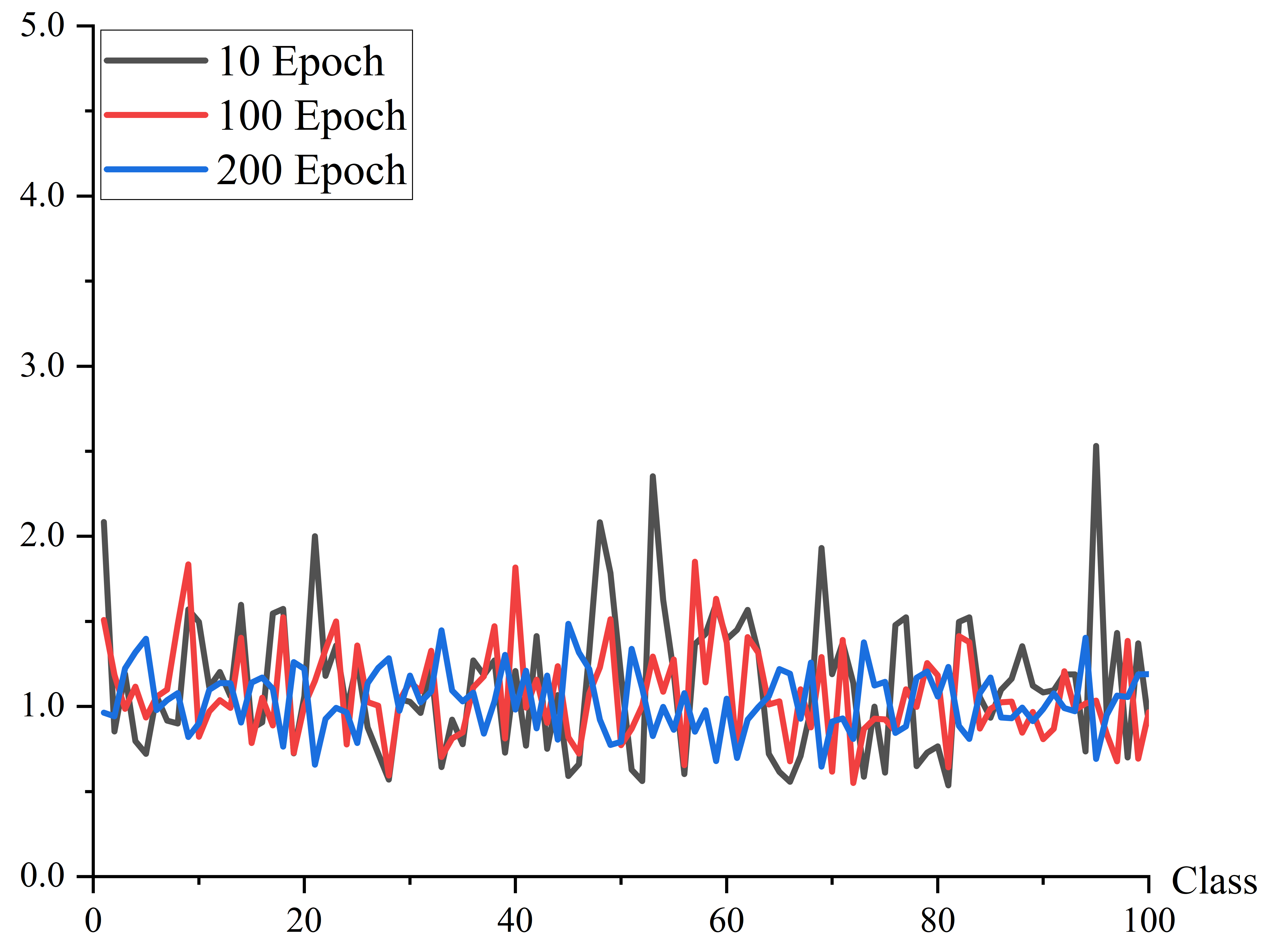}\vspace{-0.15in}
  \caption{The variance terms of different classes at three different epochs on CIFAR100.}
  \label{fig:my_label}
\end{figure}

\begin{figure}[ht]
  \centering
  \includegraphics[width=0.8\linewidth]{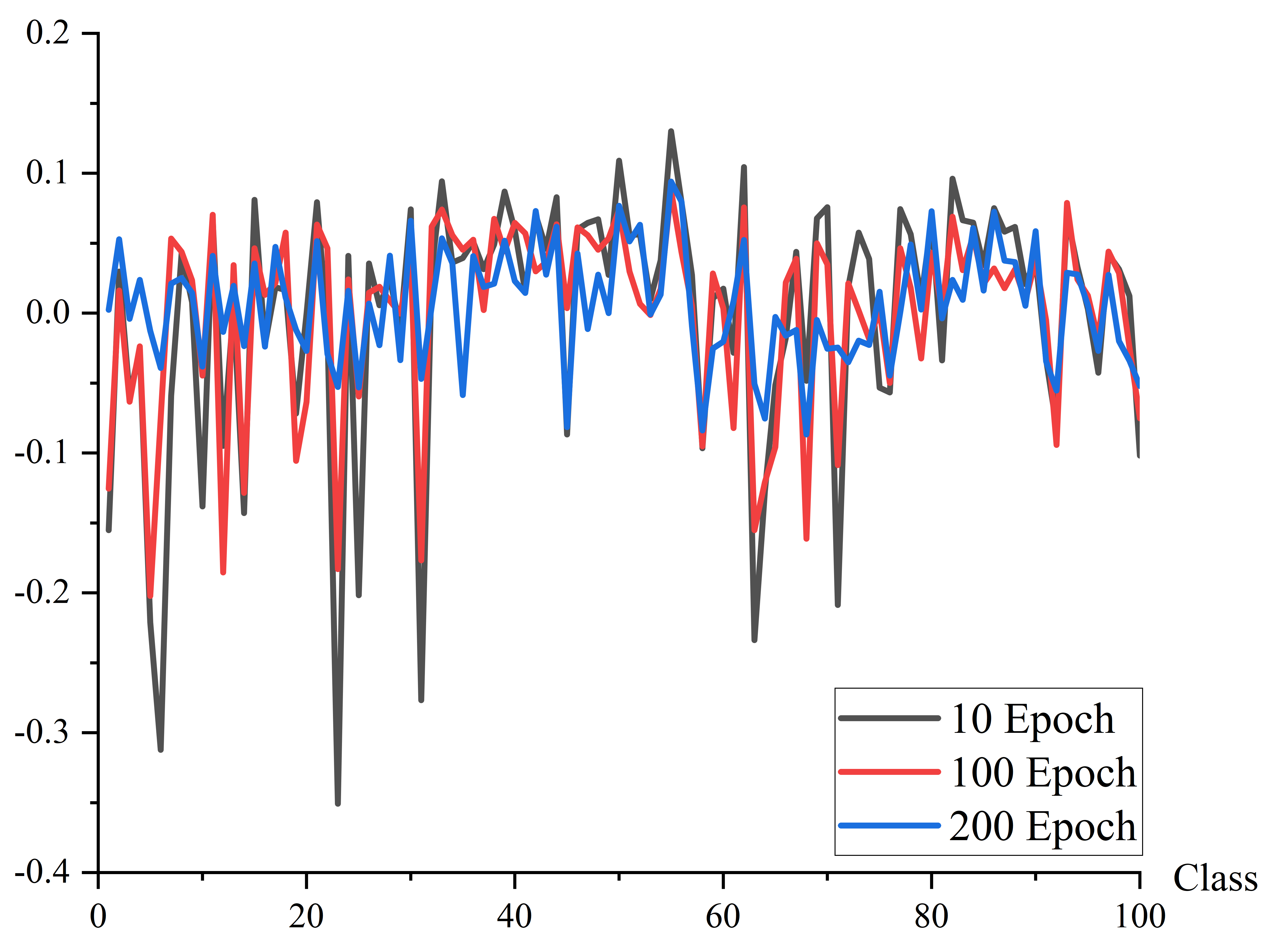}\vspace{-0.15in}
  \caption{The distance terms of different classes at three different epochs on CIFAR100-LT~(100:1).}
  \label{fig:my_label}
\end{figure}

\begin{figure}[h]
  \centering
  \includegraphics[width=0.8\linewidth]{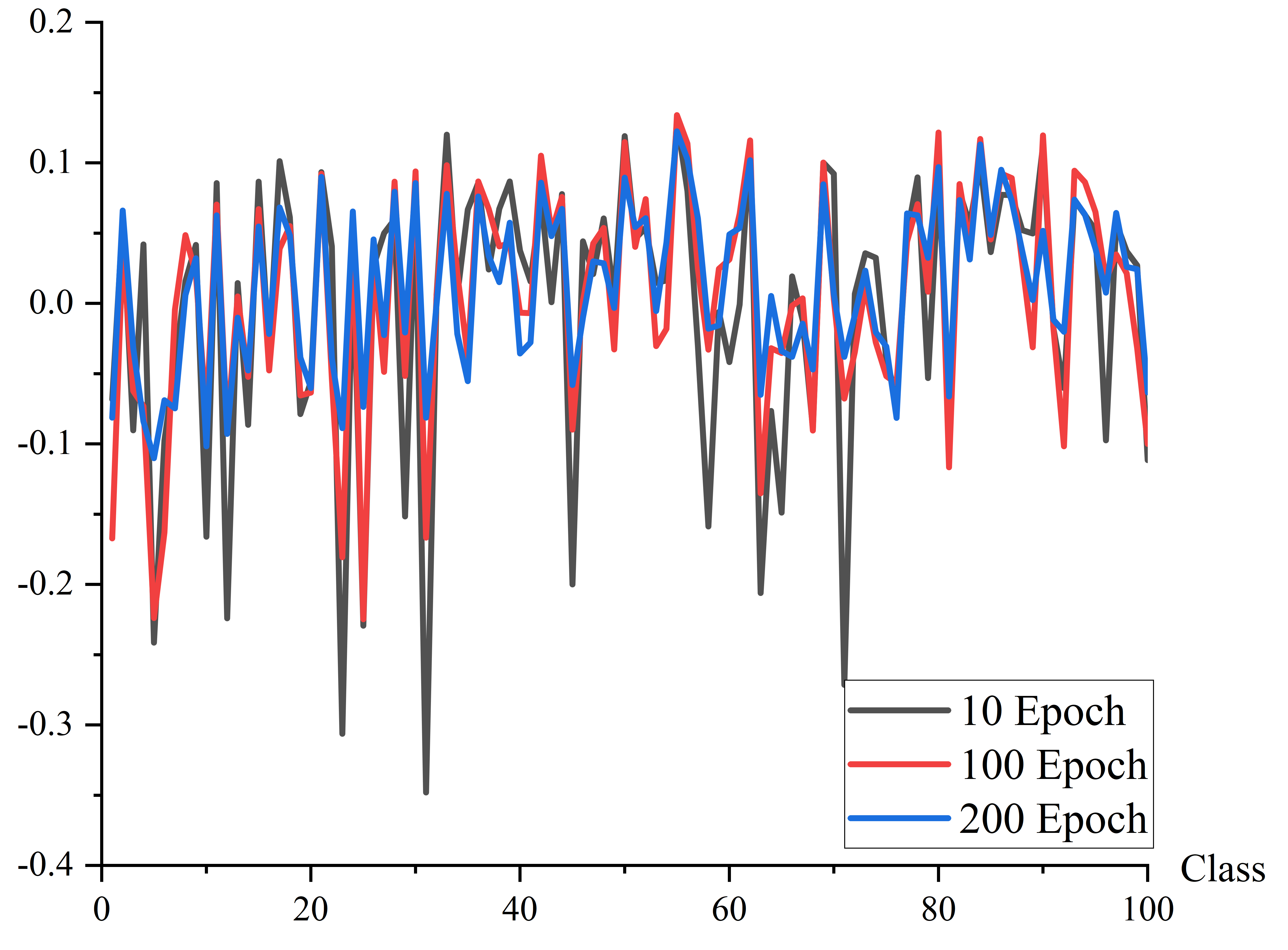}\vspace{-0.15in}
  \caption{The distance terms of different classes at three different epochs on CIFAR100.}
  \label{fig:my_label}
\end{figure}


\subsubsection{Analysis the overall logit perturbation}
As pointed out by Li et al.~\cite{LMY2022}, the loss increment/decrement incurred by logit perturbation is highly related to the positive/negative augmentation. This part investigates the loss increment/decrement incurred by the overall logit perturbation brought by the three types of logit-perturbation terms. Fig.~22 shows the loss variations incurred by MetaLAD on the CIFAR100 and CIFAR100-LT datasets. Overall, the loss increment on tail categories is larger than that of head and middle on CIFAR100-LT. Nevertheless, the curve is not monotonically increasing. This is reasonable as the class imbalance is not fully determined by the class proportion.

\begin{figure}[h] 
    \centering
    \includegraphics[width=1\linewidth]{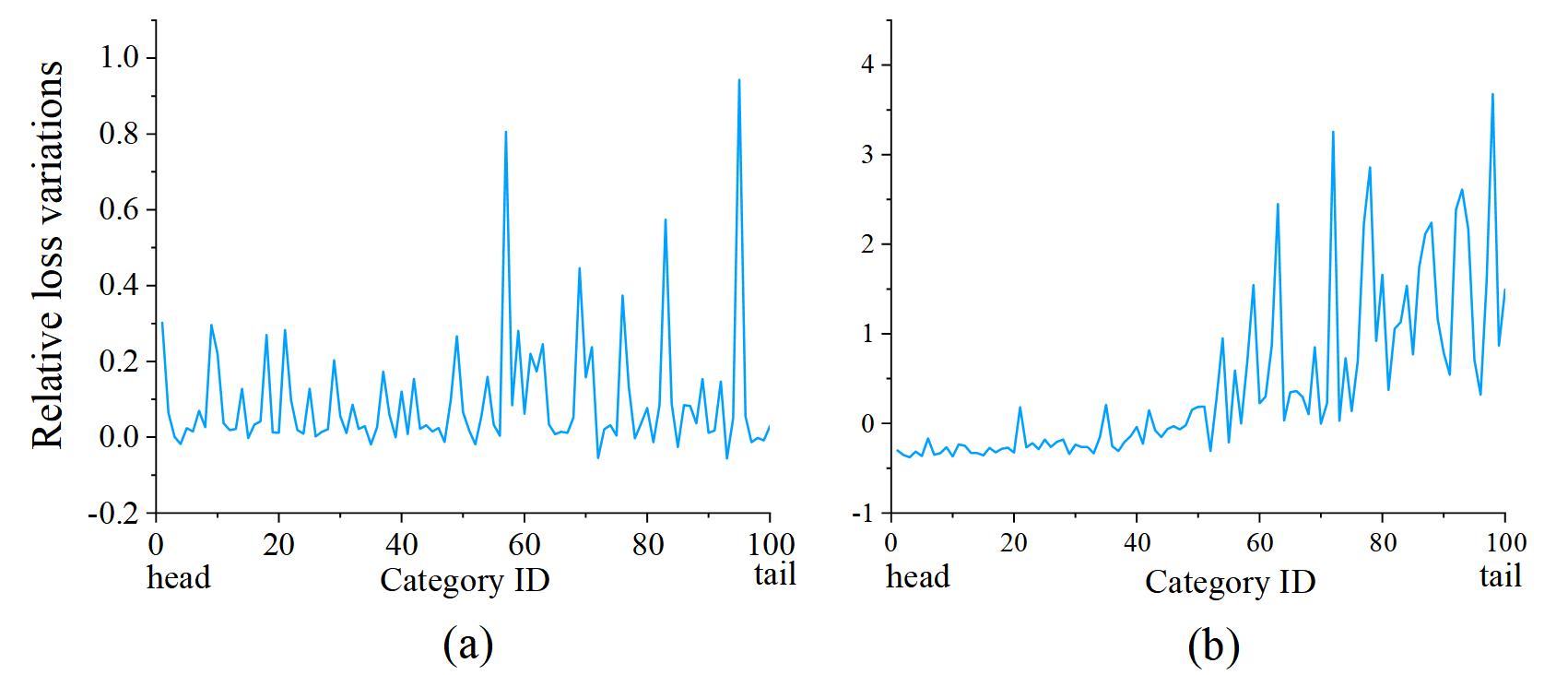} 
    \vspace{-0.3in} 
    \caption{The relative loss variations ($\frac{l'-l}{l}$) of the two datasets. (a) and (b) show the relative loss variation of MetaLAD on CIFAR100 and CIFAR100-LT, respectively. }
    \label{fig3}
     \vspace{-0.2in}
\end{figure}

\section{Conclusions}
We have revisited  the issue of class imbalance and proposed a more comprehensive taxonomy for class imbalance learning. In contrast to previous studies that focused solely on imbalanced class proportion, we have identified four additional types of imbalance: variance, distance, neighborhood, and quality. To demonstrate the significant negative effects of these new types of imbalance, we provide illustrative examples and theoretical analyses. Furthermore, we propose a new learning method called MetaLDA for situations where proportion, variance, and distance imbalance coexist. Extensive experimental results verify the effectiveness of our MetaLDA. Our future work will conduct further theoretical analyses of existing learning methods based on this new taxonomy and explore how to address local imbalance as well as neighborhood imbalance.


\bibliographystyle{IEEEtran} 
\bibliography{main.bib}

\end{document}